\documentclass[a4paper]{amsart}
\usepackage[T1]{fontenc}
\usepackage{amssymb}
\usepackage{dsfont}
\numberwithin{equation}{section}
\newtheorem{thm}{Theorem}[section]

\newtheorem{lemma}[thm]{Lemma}
\newtheorem{cor}[thm]{Corollary}
\theoremstyle{definition}
\newtheorem{dfn}{Definition}[section]

\newtheorem{exm}{Example}[section]
\theoremstyle{remark}
\newtheorem{rmk}{Remark}[section]
\allowdisplaybreaks

\title{PAC-Bayesian Bounds for Randomized Empirical Risk Minimizers}
\author[P. Alquier]{Pierre Alquier}
\address{
CREST, and
\\
Laboratoire de Probabilit\'es et Mod\`eles Al\'eatoires (Universit\'e Paris
7)
\\
175, rue du Chevaleret
\\
75252 Paris Cedex 05
\\
France.}
\urladdr{http://alquier.ensae.net/}
\date{\today}
\email{alquier@ensae.fr}
\thanks{I Would like to thank Professor Olivier Catoni for his kind help and his useful remarks.}

\begin{document}

\begin{abstract}
The aim of this paper is to generalize the PAC-Bayesian theorems
proved by Catoni \cite{Classif,manuscrit} in the classification
setting to more general problems of statistical inference.
We show how to control the deviations of the risk of randomized estimators.
A particular attention is paid to randomized estimators
drawn in a small neighborhood of classical estimators,
whose study leads to control the risk of the latter. These
results allow to bound the risk of very general estimation
procedures, as well as to perform model selection.
\end{abstract}

\maketitle

\tableofcontents

\section{Introduction}

The aim of this paper is to perform statistical inference with
observations in a possibly large dimensional space. Let us first
introduce the notations.

\subsection{General notations}

\label{generalnotations}

Let $N\in\mathds{N}^{*}$ be the number of observations. Let
$(\mathcal{Z},\mathcal{B})$ be a measurable space and $P_{1}$, ...,
$P_{N}$ be $N$ probability measures on this space, unknown to the
statistician. We assume that
$$ (Z_{1},...,Z_{N}) $$
is the canonical process on
$$ \left(\mathcal{Z}^{N},\mathcal{B}^{\otimes N}, P_{1}\otimes ...\otimes P_{N}\right) .$$
\begin{dfn}
Let us put
$$ \mathds{P} = P_{1}\otimes ... \otimes P_{N},  $$
and
$$ \overline{\mathds{P}} = \frac{1}{N}\sum_{i=1}^{N}\delta_{Z_{i}}. $$
\end{dfn}

We want to perform statistical inference on a general parameter
space $\Theta$, with respect to some loss function
$$ \ell_{\theta} : \mathcal{Z} \rightarrow \mathds{R},
\qquad  \theta \in \Theta.$$
\newcommand{\B}[1]{\mathds{#1}}

\begin{dfn}[Risk functions]
We introduce, for any $\theta\in\Theta$,
$$ r(\theta) = \overline{\mathds{P}}\left(\ell_{\theta}\right)
= \frac{1}{N} \sum_{i=1}^{N} \ell_{\theta}\left(Z_{i}\right), $$
the empirical risk function, and
$$ R(\theta) = \B{P}(\ell_{\theta}) = \frac{1}{N}
\sum_{i=1}^{N} P_{i}\left(\ell_{\theta}\right),
$$
the risk function.
\end{dfn}

We now describe three classical problems in statistics that fit the
general context described above.

\begin{exm}[Classification]
We assume that $\mathcal{Z}=\mathcal{X}\times\mathcal{Y}$ where
$\mathcal{X}$ is a set of objects and $\mathcal{Y}$ a finite set of
possible labels for these objects. Consider a set of
classification functions
$\left\{f_{\theta}:\mathcal{X}\rightarrow\mathcal{Y},\theta\in\Theta\right\}$
which assign to each object a label. Let us put, for any
$z=(x,y)\in\mathds{Z}$, $ \ell_{\theta}(z) =
\psi\left(f_{\theta}(x),y\right) $ where $\psi$ is
some symmetric discrepancy
measure. The most usual case is to use the 0-1 loss
function $ \psi(y,y') =
\delta_{y}(y') $ . If moreover
$\left|\mathcal{Y}\right|=2$ we can decide that
$\mathcal{Y}=\{-1,+1\}$ and
set $ \psi(y,y') = \mathds{1}_{\mathds{R}^{*}_{+}}(yy') $ .
However, in many practical situations, algorithmic considerations
lead to use a convex upper bound of this loss function, like
\begin{align*}
\psi(y,y') & = (1-yy')_{+} = \max(1-yy',0), \quad \text{the "hinge loss"}, \\
\psi(y,y') & = \exp(-yy'),                    \quad \text{the exponential loss}, \\
\psi(y,y') & = (1-yy')^{2},                   \quad \text{the least
square loss}.
\end{align*}
For example, Cortes and Vapnik \cite{SVM_hinge} generalized the SVM
technique to non-separable data using the hinge loss, while
Schapire, Freund, Bartlett and Lee \cite{Boosting} gave a
statistical interpretation of boosting algorithm thanks to the
exponential loss. See Zhang \cite{Zhang} for a complete study of the
performance of classification methods using these loss functions.
Remark that in this case, $f_{\theta}$ is allowed to take any real value,
and not only $-1$ or $+1$, although the labels $Y_{i}$ in the training set
are either $-1$ or $+1$.
\end{exm}

\begin{exm}[Regression estimation]
The context is the same except that the label set $\mathcal{Y}$ is
infinite, in most case it is $\mathds{R}$ or an interval of
$\mathds{R}$. Here, the most usual case is the regression with
quadratic loss, with $ \psi(y,y') = (y-y')^{2} $, however, more
general cases can be studied like the $l^{p}$ loss $ \psi(y,y') =
(y-y')^{p} $ for some $p\geq 1$.
\end{exm}

\begin{exm}[Density estimation]
Here, we assume that $P_{1}=...=P_{N}=P$ and consequently that
$\mathds{P}=P^{\otimes N}$, and we want to estimate the density
$f = d P / d \mu$ of
$P$ with respect to a known measure $\mu$. We assume
that we are given a set of probability measures $
\left\{Q_{\theta},\theta\in\Theta \right\}$ with densities
$q_{\theta}=dQ_{\theta}/d\mu$ and we use the loss function
$\ell_{\theta}(z)=-\log\left[q_{\theta}(z)\right]$. Indeed in this case,
we can write under suitable hypotheses
\begin{multline*}
R(\theta) = P\left(-\log \circ \, q_{\theta} \right) = P\left(-\log
\circ \frac{dQ_{\theta}}{d\mu} \right) = P\left(\log \circ
\frac{dP}{dQ_{\theta}} \right) + P\left(\log \circ \, \frac{d\mu}{dP}
\right)
\\
= \mathcal{K}\left(P,Q_{\theta}\right) -  P\left(\log \circ f\right),
\end{multline*}
showing that the risk is the Kullback-Leibler divergence between $P$ and $Q_{\theta}$
up to a constant (the definition of $\mathcal{K}$ is reminded in
this paper, see Definition \ref{KLdef} page \pageref{KLdef}).
\end{exm}

In each case the objective is to estimate $\arg\min R$ on the basis
of the observations $Z_{1}$, ..., $Z_{N}$ - presumably using in some
way or another the value of the empirical risk.
We have to notice that when the space $\Theta$ is
large or complex (for example a vector space with large dimension),
$\arg\min R$ and $\arg\min r$ can be very different. This does not
happen if $\Theta$ is simple (for example a vector space with small
dimension), but such a case is less interesting as we have to
eliminate a lot of dimensions in $\Theta$ before proceeding to
statistical inference with no guarantees that these directions are
not relevant.

\subsection{Statistical learning theory and PAC-Bayesian point of view}

The learning theory point of view introduced by Vapnik and
Cervonenkis (\cite{VC}, see Vapnik \cite{Vapnik} for a presentation
of the main results in English) gives a setting that proved to be
adapted to deal with estimation problems in large dimension. This
point of view received an important interest over the past few
years, see for example the well-known books of Devroye, Gy�rfi and
Lugosi \cite{Lugosi}, Friedman, Hastie and Tibshirani
\cite{Elements} or more recently the paper by Boucheron, Bousquet
and Lugosi \cite{Boucheron} and the references therein, for a state
of the art.

The idea of Vapnik and Cervonenkis is to introduce a structure,
namely a family of submodels $\Theta_{1}$, $\Theta_{2}$, ... The
problem of model selection then arises: we must choose the submodel
$\Theta_{k}$ in which the minimization of the empirical risk $r$
will lead to the smallest possible value for the real risk $R$. This
choice requires to estimate the complexity of submodels
$\Theta_{k}$. An example of complexity measure is the so-called
Vapnik Cervonenkis dimension or VC-dimension, see \cite{VC,Vapnik}.

The PAC-Bayesian point of view, introduced in the context of
classification by McAllester \cite{McAllester,McA2} is based on the
following remark: while classical measures of complexity (like
VC-dimension) require theoretical results on the submodels, the
introduction of a probability measure $\pi$ on the model $\Theta$
allows to measure empirically the complexity of every submodel. In a
more technical point of view, we will see later that $\pi$ allows a
generalization of the so-called union bound (see \cite{McA2} for
example). This point of view might be compared with Rissanen's work
on MDL (Minimum Description Length, see \cite{Rissanen}) making a
link between statistical inference and information theory, and
$-\log\pi({\theta})$ can be seen as the length of a code for the
parameter $\theta$ (at least when $\Theta$ is finite).

The PAC-Bayesian point of view was developed in more contexts
(classification, least square regression and density estimation) by
Catoni \cite{Cat7}, and then improved in the context of
classification by Catoni \cite{Classif}, Audibert \cite{Audibert}
and in the context of least square regression by Audibert
\cite{AudibertReg} and of regression with a general loss in our PhD
thesis \cite{Alquier2006}. The most recent work in the context of
classification by Catoni \cite{manuscrit} improves the upper-bound
given on the risk of the PAC-Bayesian estimators, leading to purely
empirical bounds that allow to perform model selection with no
assumption on the probability measure $\mathds{P}$. The aim of this
work is to extend these results to the very general context of
statistical inference introduced in subsection
\ref{generalnotations}, that includes classification, regression
with a general loss function and density estimation.

Let us introduce our estimators.

\begin{dfn}
Let us assume that we have a family of functions
$$
\psi_{\theta}^{i}:\mathcal{Z}\rightarrow\mathds{R} \cup \{
+ \infty \}
$$ indexed by $i$
in a finite or countable set $I$ and by $\theta\in\Theta$. For every
$i\in I$ we choose:
$$ \hat{\theta}_{i}\in \arg\min_{\theta\in\Theta} \overline{\mathds{P}}\left(\psi_{\theta}^{i}\right) .$$
\end{dfn}

\begin{exm}[Empirical risk minimization and model selection]
If we take $I=\{0\}$ we can choose
$\psi^{0}_{\theta}(z)=l_{\theta}(z)$ and we obtain
$\overline{\mathds{P}}\left(\psi_{\theta}^{0}\right)=r(\theta)$ and
so
$$ \hat{\theta}^{0}=\arg\min_{\theta\in\Theta} r(\theta) $$ the
empirical risk minimizer. In the case where the dimension of
$\Theta$ is large, we can choose several submodels indexed by a
finite or countable family $I$: $(\Theta_{i},i\in I)$. In order to
obtain
$$ \hat{\theta}^{i}=\arg\min_{\theta\in\Theta_{i}}r(\theta)  $$
we can put
$$ \psi_{\theta}^{i}(.) = \left\{\begin{array}{l}
l_{\theta}(.) \quad \text{if} \quad \theta\in\Theta_{i}
\\ \\
+\infty \quad \text{otherwise}.
\end{array}\right. $$
The problem of the selection of the $\hat{\theta}_{i}$ with the
smallest possible risk (so-called model selection problem) can be
solved with the help of PAC-Bayesian bounds.
\end{exm}

Note that PAC-Bayesian bounds given by Catoni
\cite{Classif,Cat7,manuscrit} usually apply to "randomized
estimators". More formally, let us introduce a $\sigma$-algebra
$\mathcal{T}$ on $\Theta$ and a probability measure $\pi$ on the
measurable space $(\Theta,\mathcal{T})$. We will need the following
definitions.

\begin{dfn}
For any measurable set $(E,\mathcal{E})$, we let
$\mathcal{M}_{+}^{1}(E)$ denote the set of all probability measures
on the measurable space $(E,\mathcal{E})$.
\end{dfn}

\begin{dfn}
In order to generalize the notion of estimator (a measurable
function $\mathcal{Z}^{N}\rightarrow\Theta$), we call a randomized
estimator any function
$\rho:\mathcal{Z}^{N}\rightarrow\mathcal{M}_{+}^{1}(\Theta) $ that
is a regular conditional probability measure. For the sake of
simplicity, the sample being given, we will write $\rho$ instead of
$\rho\,(Z_{1},...,Z_{N})$.
\end{dfn}

PAC-Bayesian bounds for randomized estimators are usually
given for their mean risk
$$ \int_{\theta\in\Theta} R(\theta) d\rho(\theta), $$
whereas here we will rather focus on
$R(\tilde{\theta})$, where $\tilde{\theta}$ is drawn from $\rho$ and
$\rho$ is highly concentrated around a "classical" (deterministic)
estimator $\hat{\theta}_{i}$.

\subsection{Truncation of the risk}

In this subsection, we introduce a truncated version of the relative
risk of two parameters $\theta$ and $\theta'$.

\begin{dfn}
We put, for any $\lambda\in\mathds{R}_{+}^{*}$ and
$(\theta,\theta')\in\Theta^{2}$
$$ R_{\lambda}(\theta,\theta') = \mathds{P}\left[\left(\ell_{\theta}-\ell_{\theta'}\right)\wedge\frac{N}{\lambda}\right].$$
Note of course that if $\mathds{P}$-almost surely, we have
$\ell_{\theta}-\ell_{\theta'} \leq N/\lambda$ then
$R_{\lambda}(\theta,\theta') = R(\theta) - R(\theta') $.
\end{dfn}

In what follows, we will give empirical bounds on
$R_{\lambda}(\theta,\theta')$ for some $\theta$ and $\theta'$ chosen
by some statistical procedure. One can wonder why we prefer to bound
this truncated version of the risk instead of $R(\theta) -
R(\theta')$. The reason is the following. In this paper, we want to
give bounds that hold with no particular assumption on the unknown
data distribution $\mathds{P}$. However, it is clear that we cannot
obtain a purely empirical bound on $R(\theta) - R(\theta')$ with no
assumption on the data distribution, as it is shown by the following
example.

\begin{exm}
Let us choose $c>0$ and $\lambda>0$. We assume that
$P_{1}=...=P_{N}$ and that $\Theta=\{\theta,\theta'\}$ with
$l_{\theta'}(z)=0$. We put $l_{\theta}(Z)=cN$ with
probability $1/N$ and $0$ otherwise. Then we have $R(\theta')=0$
and
$$ R(\theta) = \frac{1}{N}cN+\left(1-\frac{1}{N}\right)0 = c $$
while $r(\theta')=0$ and with probability at least $(1-1/N)^{N}
\simeq \exp( - 1 )$ we
also have $r(\theta)=0$, this means that we cannot upper bound
precisely $R(\theta) - R(\theta')$ by empirical quantities with no
assumption.
\end{exm}

So, we introduce the truncation of the risk. However, two remarks
shall be made. First, in the case of a bounded loss function $\ell$,
with a large enough ratio $N/\lambda$ we have $R_{\lambda}(\theta,\theta') =
R(\theta) - R(\theta') $.

In the general case, if we want to upper bound $R(\theta) -
R(\theta') $ we can make additional hypotheses on the data
distribution, ensuring that we can dispose of a (known) upper-bound :
$$  \Delta_{\lambda}(\theta,\theta') \geq R(\theta) - R(\theta') - R_{\lambda}(\theta,\theta') $$
 as it is done in our PhD Thesis \cite{Alquier2006}. For the sake of completeness, such an upper
bound is given in the Appendix, page \pageref{annex}.

\subsection{Main tools}

In this subsection, we give two lemmas that will be useful in order
to build PAC-Bayesian theorems. First, let us recall the following
definition. In this whole subsection, we assume that
$(E,\mathcal{E})$ is an arbitrary measurable space.

\begin{dfn}
For any measurable function $h:E\rightarrow\mathds{R}$, for any
measure $m\in\mathcal{M}_{+}^{1}(E)$ we put
$$ m(h) = \sup_{B\in\mathds{R}} \int_{E} \left[h(x) \wedge B \right]m(dx).$$
\end{dfn}

\begin{dfn}[Kullback-Leibler divergence] \label{KLdef}
Given a measurable space $(E,\mathcal{E})$, we define , for any
$(m,n)\in[\mathcal{M}_{+}^{1}(E)]^{2}$,
the Kullback-Leibler divergence function
$$
\mathcal{K}(m,n)=
\begin{cases}
\displaystyle \int_{E} dm(e) \biggl\{ \log \biggl[ \frac{d m}{d n}(e ) \biggr] \biggr\} & \text{ if } m \ll n,
\\
+ \infty & \text{ otherwise.}
\end{cases}
$$
\end{dfn}

\begin{lemma}[Legendre transform of the Kullback divergence function] \label{LEGENDRE}
For any $n\in\mathcal{M}_{+}^{1}(E)$, for any measurable function
$h:E\rightarrow\mathds{R}$ such that $ n (\exp \circ h) <+\infty $
we have
\begin{equation} \label{lemmacatoni}
\log n (\exp\circ h)=\sup_{m \in\mathcal{M}_{+}^{1}(E)}\biggl(m
(h)-\mathcal{K}(m,n)\biggr),
\end{equation}
where by convention $\infty-\infty=-\infty$. Moreover, as soon as $h$ is
upper-bounded on the support of $n$, the supremum with respect to
$m$ in the right-hand side is reached for the Gibbs distribution,
$n_{\exp(h)}$ given by:
$$ \forall e \in E,\quad \frac{d n_{\exp(h)}}{d n}(e)=\frac{\exp[h(e)]}{\pi(\exp \circ h)}.$$
\end{lemma}

The proof of this lemma is given at the end of the paper, in a
section devoted to proofs (subsection \ref{PROOFLEGENDRE} page
\pageref{PROOFLEGENDRE}). We now state another lemma that will be
useful in the sequel. First, we need the following
definition.

\begin{dfn} \label{PHIfunc}
We put, for any $\alpha\in\mathds{R}^{*}_{+}$,
\begin{align*}
\Phi_{\alpha} : \quad \left]-\infty,1/\alpha\right[ &  \rightarrow \mathds{R} \\
                                   t &  \mapsto     -\frac{\log\left(1-\alpha t \right)}{\alpha}.
\end{align*}
\end{dfn}

Note that $\Phi_{\alpha}$ is invertible, that for any
$u\in\mathds{R}$,
$$ \Phi_{\alpha}^{-1}(u) = \frac{1-\exp\left(-\alpha u\right)}{\alpha} \leq u,$$
and that $\frac{2(\Phi_{\alpha}(x)-x)}{\alpha x^{2}}
\xrightarrow[x\rightarrow 0]{} 1 $. Also note that for $\alpha>0$,
$\Phi_{\alpha}$ is convex and that $\Phi_{\alpha}(x)\geq x$. An
elementary study of this function also proves that for any $C>0$, for
any $\alpha\in\left]0,1/(2C)\right[$ and any $p\in[0,C]$ we have:
$$ \Phi_{\alpha}(p) \leq p + \frac{\alpha p^{2}}{2} .$$
We can now give the lemma.

\begin{lemma} \label{deviationlemma}
We have, for any $\lambda\in\mathds{R}_{+}^{*}$, for any
$a\in]0,1]$, for any $(\theta,\theta')\in\Theta^{2}$,
$$ \mathds{P} \exp\left\{ \lambda \Phi_{\frac{\lambda}{N}}\left[ R_{\frac{\lambda}{a}}\left(\theta,\theta'\right)\right]
-\frac{\lambda}{N}\sum_{i=1}^{N}
\Phi_{\frac{\lambda}{N}}\left[\left(\ell_{\theta}-\ell_{\theta'}\right)(Z_{i})\wedge\frac{aN}{\lambda}\right]
\right\} = 1.$$
\end{lemma}

The proof is almost trivial, we give it now in order to emphasize
the role of the truncation and of the change of variable.

\begin{proof}
For any $\lambda\in\mathds{R}_{+}^{*}$, for any
$(\theta,\theta')\in\Theta^{2}$,
\begin{multline*}
 \mathds{P} \exp\left\{ \lambda \Phi_{\frac{\lambda}{N}}\left[ R_{\frac{\lambda}{a}}\left(\theta,\theta'\right)\right]
-\frac{\lambda}{N}\sum_{i=1}^{N}
\Phi_{\frac{\lambda}{N}}\left[\left(\ell_{\theta}-\ell_{\theta'}\right)(Z_{i})\wedge\frac{aN}{\lambda}\right]
\right\}
\\
= \mathds{P} \exp\Biggl\{ \sum_{i=1}^{N} \Biggl(\log
\left[1-\frac{\lambda}{N}\left((l_{\theta}-l_{\theta'})(Z_{i})\wedge\frac{aN}{\lambda}\right)\right]
\\ \shoveright{
-\log \left[1- \frac{\lambda}{N}P_{i}
\left((l_{\theta}-l_{\theta'})(Z_{i})\wedge\frac{aN}{\lambda}\right)
\right] \Biggr)\Biggr\} }
\\
= \mathds{P} \left[ \prod_{i=1}^{N}
\frac{1-\frac{\lambda}{N}\left((l_{\theta}-l_{\theta'})(Z_{i})\wedge\frac{aN}{\lambda}\right)}
{1-\frac{\lambda}{N}P_{i}\left((l_{\theta}-l_{\theta'})(Z_{i})\wedge\frac{aN}{\lambda}\right)}
\right]
\\
=  \prod_{i=1}^{N} P_{i} \left[
\frac{1-\frac{\lambda}{N}\left((l_{\theta}-l_{\theta'})(Z_{i})\wedge\frac{aN}{\lambda}\right)}
{1-\frac{\lambda}{N}P_{i}\left((l_{\theta}-l_{\theta'})(Z_{i})\wedge\frac{aN}{\lambda}\right)}
\right] =1 .
\end{multline*}
\end{proof}

Note that this lemma will be used as an alternative to Hoeffding's
or Bernstein's (see \cite{Hoeffding,Bernstein}) inequalities in order
to prove PAC inequalities.

\subsection{A basic PAC-Bayesian Theorem}

Let us integrate Lemma \ref{deviationlemma} with respect to
$(\theta,\theta')$ with a given probability measure
$n=\pi\otimes\pi'$ with
$(\pi,\pi')\in[\mathcal{M}_{+}^{1}(\Theta)]^{2}$. Applying
Fubini-Tonelli Theorem we obtain:
\begin{multline} \label{intermediaire}
\mathds{P} \Biggl\{\int_{(\theta,\theta')\in\Theta^{2}}
d ( \pi\otimes \pi') (\theta,\theta') \exp\Biggl\{
\lambda \Phi_{\frac{\lambda}{N}}\left[
R_{\frac{\lambda}{a}}\left(\theta,\theta'\right)\right]
\\
-\frac{\lambda}{N}\sum_{i=1}^{N}
\Phi_{\frac{\lambda}{N}}\left[\left(\ell_{\theta}-\ell_{\theta'}\right)(Z_{i})\wedge\frac{aN}{\lambda}\right]
\Biggr\}
\Biggr\} = 1.
\end{multline}
This implies that for any
$(\rho,\rho')\in[\mathcal{M}_{+}^{1}(\Theta)]^{2}$,
\begin{multline*}
\mathds{P} \Biggl\{\int_{(\theta,\theta')\in\Theta^{2}}
d( \rho\otimes \rho') (\theta,\theta') \exp\Biggl\{
\lambda \Phi_{\frac{\lambda}{N}}\left[
R_{\frac{\lambda}{a}}\left(\theta,\theta'\right)\right]
\\
-\frac{\lambda}{N}\sum_{i=1}^{N}
\Phi_{\frac{\lambda}{N}}\left[\left(\ell_{\theta}-\ell_{\theta'}\right)(Z_{i})\wedge\frac{aN}{\lambda}\right]
-\log\left[\frac{d(\rho\otimes\rho')}{d(\pi\otimes\pi')}(\theta,\theta')\right]
\Biggr\}
\Biggr\} \leq 1.
\end{multline*}
(This inequality becomes an equality when $\pi \ll \rho$ and $\pi' \ll \rho'$.)

\begin{thm}
\label{BASICthmdes} Let us assume that we have
$(\pi,\pi')\in\mathcal{M}_{+}^{1}(\Theta)^{2}$, and two randomized
estimators $\rho$ and $\rho'$. For any $\varepsilon>0$, for any
$(a,\lambda)\in]0,1]\times\mathds{R}_{+}^{*}$, with
$\mathds{P} (\rho\otimes\rho')$-probability at least $1-\varepsilon$
over the sample $(Z_{i})_{i=1,...,N}$ and the parameters
$(\tilde{\theta},\tilde{\theta'})$, we have:
\begin{multline*}
R_{\frac{\lambda}{a}}\left(\tilde{\theta},\tilde{\theta}'\right)
 \leq
\Phi_{\frac{\lambda}{N}}^{-1} \Biggl\{ \frac{1}{N}\sum_{i=1}^{N}
\Phi_{\frac{\lambda}{N}}\left[\left(\ell_{\tilde{\theta}}-\ell_{\tilde{\theta}'}\right)(Z_{i})\wedge\frac{aN}{\lambda}\right]
\\
+ \frac{\log
\left[\frac{d\rho}{d\pi}\left(\tilde{\theta}\right)\right]+\log
\left[\frac{d\rho'}{d\pi'}\left(\tilde{\theta}'\right)\right]+\log\frac{1}{\varepsilon}}{\lambda}
\Biggr\}.
\end{multline*}
\end{thm}

In order to provide an interpretation of Theorem \ref{BASICthmdes},
let us give the following corollary in the bounded case, which is
obtained using basic properties of the function $\Phi$ given just
after Definition \ref{PHIfunc} page \pageref{PHIfunc}. In this case,
the parameter $a$ is just set to $1$.

\begin{cor}
\label{cor1} Let us assume that for any
$(\theta,z)\in\Theta\times\mathcal{Z},0<l_{\theta}(z)<C$. Let us
assume that we have $(\pi,\pi')\in\mathcal{M}_{+}^{1}(\Theta)^{2}$,
and two randomized estimators $\rho$ and $\rho'$. For any
$\varepsilon>0$, for any $\lambda\in]0,N/(2C)]$, with
$\mathds{P} (\rho\otimes\rho')$-probability at least $1-\varepsilon$ we
have:
\begin{multline*}
R \bigl(\tilde{\theta}\bigr)-R\bigl(\tilde{\theta}'\bigr)
 \leq
\Phi_{\frac{\lambda}{N}}^{-1} \Biggl\{
r\bigl(\tilde{\theta}\bigr)-r\bigl(\tilde{\theta}'\bigr) +
\frac{\lambda}{2N}\overline{\mathds{P}}\left[\left(l_{\tilde{\theta}}-l_{\tilde{\theta}'}\right)^{2}\right]
\\
+ \frac{\log
\left[\frac{d\rho}{d\pi}\bigl(\tilde{\theta}\bigr)\right]+\log
\left[\frac{d\rho'}{d\pi'}\bigl(\tilde{\theta}'\bigr)\right]+\log\frac{1}{\varepsilon}}{\lambda}
\Biggr\}.
\end{multline*}
\end{cor}

We can see that the difference of the "true" risk of the randomized
estimators $\tilde{\theta}$ and $\tilde{\theta}'$, drawn
independently from $\rho$ and $\rho'$, is upper bounded by the
difference of the empirical risk, plus a variance term and a
complexity term expressed in terms of the $\log$ of the density of
the randomized estimator with respect to a given prior. So Theorem
\ref{BASICthmdes} provides an empirical way to compare the
theoretical performance of two randomized estimators, leading to
applications in model selection. This paper is devoted to
improvements of Theorem \ref{BASICthmdes} (we will see in the sequel
that this theorem does not necessarily lead to optimal estimators)
and to the effective construction of estimators using variants of
Theorem \ref{BASICthmdes}.

Now, note that the choice of the randomized estimators $\rho$ and
$\rho'$ is not straightforward. The following theorem, which gives an
integrated variant of Theorem \ref{BASICthmdes}, can be usefull for
that purpose.

\begin{thm} \label{BASICthmrel}
Let us assume that we have
$(\pi,\pi')\in\mathcal{M}_{+}^{1}(\Theta)^{2}$. For any
$\varepsilon>0$, for any
$(a,\lambda)\in]0,1]\times\mathds{R}_{+}^{*}$, with
$\mathds{P}$-probability at least $1-\varepsilon$, for any
$(\rho,\rho') \in\mathcal{M}_{+}^{1}(\Theta)^{2}$,
\begin{multline*}
\int_{\Theta^{2}} R_{\frac{\lambda}{a}}(\theta,\theta')
d(\rho\otimes\rho')(\theta,\theta')
\\
 \leq
\Phi_{\frac{\lambda}{N}}^{-1} \Biggl\{ \int_{\Theta^{2}}
\frac{1}{N}\sum_{i=1}^{N}
\Phi_{\frac{\lambda}{N}}\left[\left(\ell_{\theta}-\ell_{\theta'}\right)(Z_{i})\wedge\frac{aN}{\lambda}\right]
d\,(\rho\otimes\rho')(\theta,\theta')
\\
+
\frac{\mathcal{K}(\rho,\pi)+\mathcal{K}(\rho',\pi')+\log\frac{1}{\varepsilon}}{\lambda}
\Biggr\}.
\end{multline*}
\end{thm}

The proof is given in subsection \ref{proofbasic} page
\pageref{proofbasic}.

\subsection{Main results of the paper}

In our PhD dissertation \cite{Alquier2006}, a particular case of
Theorem \ref{BASICthmrel} is given and applied to regression
estimation with quadratic loss in a bounded model of finite
dimension $d$. In this particular case, it is shown that the
estimators based on the minimization of the right-hand side of
Theorem \ref{BASICthmrel} do not achieve the optimal rate of
convergence: $d/N$, but only $(d\log N)/N$. A solution is given by
Catoni in \cite{Cat7} and consists in replacing the prior $\pi$ by
the so-called "localized prior" $\pi_{\exp(-\beta R)}$ for a given
$\beta>0$. The main problem is that this choice leads to the
presence of non-empirical terms in the right-hand side,
$\mathcal{K}(\rho,\pi_{\exp(-\beta R)})$.

In Section \ref{empirical}, we give an empirical bound for this term
$\mathcal{K}(\rho,\pi_{\exp(-\beta R)})$. We also give a heuristic
that leads to this technique of localization.

In Section \ref{application}, we show how this result, combined with
Theorem \ref{BASICthmrel}, leads to the effective construction of an
estimator that can reach optimal rates of convergence.

The proofs of the theorems stated in this paper are gathered in
Section \ref{proofs}.

\section{Empirical bound for the localized complexity and localized PAC-Bayesian theorems}

\label{empirical}

\subsection{Mutual information between the sample and the parameter}

Let us consider Theorem \ref{BASICthmrel} with
$\rho'=\pi'=\delta_{\theta'}$ for a given parameter $\theta'$. For
the sake of simplicity, let us assume in this subsection that we are
in the bounded case ($l_{\theta}$ bounded by $C$). Theorem
\ref{BASICthmrel} ensures that, for any $\lambda\in]0,N/(2C)[$, with
$\mathds{P}$-probability at least $1-\varepsilon$, for any
$\rho \in\mathcal{M}_{+}^{1}(\Theta)$,
$$
\rho\left(R\right)-R(\theta')
 \leq
\rho\left(r\right)-r(\theta') +
\frac{\lambda}{2N}\overline{\mathds{P}}\left[
\int_{\Theta}\left(l_{\theta}-l_{\theta'}\right)^{2}d\rho(\theta)\right]
+ \frac{\mathcal{K}(\rho,\pi)+\log\frac{1}{\varepsilon}}{\lambda}.
$$

This is an incitation to choose
$$ \rho=\arg\min_{\mu\in\mathcal{M}_{+}^{1}(\Theta)}
\left[\mu\left(r\right)+
\frac{\lambda}{2N}\overline{\mathds{P}}\left[
\int_{\Theta}\left(l_{\theta}-l_{\theta'}\right)^{2}d\rho(\theta)\right]+
\frac{\mathcal{K}(\mu,\pi)}{\lambda} \right].
$$
However, if we choose to neglect the variance term, we may consider the
following randomized estimator:
$$ \rho=\arg\min_{\mu\in\mathcal{M}_{+}^{1}(\Theta)}
\left[\mu\left(r\right) + \frac{\mathcal{K}(\mu,\pi)}{\lambda}
\right].
$$
Actually, in this case, Lemma \ref{LEGENDRE} leads to:
$$ \rho = \pi_{\exp(-\lambda r)} .$$

Let us remark that, for any
$(\rho,\pi)\in\mathcal{M}_{+}^{1}(\Theta)$ we have:
\begin{equation} \label{divergence}
P \Bigl[\mathcal{K}(\rho,\pi)\Bigr] = P\Bigl[
\mathcal{K}(\rho,P(\rho))\Bigr] + \mathcal{K}(P(\rho),\pi).
\end{equation}
This implies that, for a given data-dependent $\rho$, the optimal
deterministic measure $\pi$ is $P(\rho)$  in the sense that it
minimizes the expectation of $\mathcal{K}(\rho,\pi)$ (left-hand side
of Equation \ref{divergence}), making it equal to the expectation of
$\mathcal{K}(\rho,P(\rho))$. This last quantity is the mutual
information between the estimator and the sample.

So, for $\rho=\pi_{\exp(-\lambda r)}$, this is an incitation to
replace the prior $\pi$ with $P\left(\pi_{\exp(-\lambda r)}\right)$.
It is then natural to approximate this distribution by
$\pi_{\exp(-\lambda R)}$.

In what follows, we replace $\pi$ by $\pi_{\exp(-\beta R)}$ for a
given $\beta>0$, keeping one more degree of freedom. Now, note that
Theorem \ref{BASICthmrel} gives:
\begin{multline*}
\rho\left(R\right)-R(\theta')
\\
 \leq
\rho\left(r\right)-r(\theta') +
\frac{\lambda}{2N}\overline{\mathds{P}}\left[
\int_{\Theta}\left(l_{\theta}-l_{\theta'}\right)^{2}d\rho(\theta)\right]
+ \frac{\mathcal{K}\left(\rho,\pi_{\exp(-\beta
R)}\right)+\log\frac{1}{\varepsilon}}{\lambda}
\end{multline*}
and note that the upper bound is no longer empirical (observable to
the statistician).

The aim of the next subsection is to upper bound
$\mathcal{K}\left(\rho,\pi_{\exp(-\beta R)}\right)$ by an empirical
bound in a general setting.

\subsection{Empirical bound of the localized complexity}

\begin{dfn}
Let us put, for any $(a,\lambda)\in]0,1]\times\mathds{R}_{+}^{*}$
and $(\theta,\theta')\in\Theta^{2}$,
$$ v_{a,\frac{\lambda}{N}}(\theta,\theta') = \frac{2N}{\lambda} \left\{\frac{\lambda}{N}\sum_{i=1}^{N}
\Phi_{\frac{\lambda}{N}}\left[\left(\ell_{\theta}-\ell_{\theta'}\right)(Z_{i})\wedge\frac{aN}{\lambda}\right]-\Bigl[r(\theta)-r(\theta')\Bigr]\right\}.
$$
\end{dfn}

\begin{thm}
\label{entropy1} Let us choose a distribution
$\pi\in\mathcal{M}_{+}^{1}(\Theta)$. For any $\varepsilon>0$, for
any
$(a,\gamma,\beta)\in]0,1]\times\mathds{R}_{+}^{*}\times\mathds{R}_{+}^{*}$
such that $\beta<\gamma$, with $\mathds{P}$-probability at least
$1-\varepsilon$, for any $\rho\in\mathcal{M}_{+}^{1}(\Theta)$,
$$
\mathcal{K}\left(\rho,\pi_{\exp(-\beta R)}\right) \leq
\mathcal{BK}_{a,\beta,\gamma}(\rho,\pi) +
\frac{\beta}{\gamma-\beta}\log\frac{1}{\varepsilon}
$$
where
\begin{multline*}
\mathcal{BK}_{a,\beta,\gamma}(\rho,\pi) =
\left(1-\frac{\beta}{\gamma}\right)^{-1}\biggl\{
\mathcal{K}\left(\rho,\pi_{\exp(-\beta r)}\right)
\\
+ \log \int_{\Theta} \pi_{\exp(-\beta r)}(d\theta') \exp
\left[\int_{\Theta}\rho(d\theta) \left(\frac{\beta\gamma}{2N}
v_{a,\frac{\gamma}{N}}(\theta,\theta')+\beta\Delta_{\frac{\gamma}{a}}
(\theta,\theta')\right) \right] \biggr\}.
\end{multline*}
\end{thm}

The proof is given in the section dedicated to proofs, more
precisely in subsection \ref{PROOFe1} page \pageref{PROOFe1}. Note
that the localized entropy term is controlled by its empirical
counterpart together with a variance term.

Before combining this result with Theorem \ref{BASICthmrel}, we give
the analogous result for the non-integrated case, which proof is
also given in subsection \ref{PROOFe1}.

\begin{thm}
\label{entropy2} Let us choose a distribution
$\pi\in\mathcal{M}_{+}^{1}(\Theta)$ and a randomized estimator
$\rho$. For any $\varepsilon>0$ and $\eta>0$, for any
$(a,\gamma,\beta)\in]0,1]\times\mathds{R}_{+}^{*}\times\mathds{R}_{+}^{*}$
such that $\beta<\gamma$, with $\mathds{P}\rho$-probability at least
$1-\varepsilon$,
$$
\log\left[\frac{d\rho}{d\pi_{\exp[-\beta R]}}(\tilde{\theta})\right]
\leq \mathcal{D}_{a,\beta,\gamma}(\rho,\pi)(\tilde{\theta}) +
\frac{\beta}{\gamma-\beta}\log\frac{1}{\varepsilon}
$$
where
\begin{multline*}
\mathcal{D}_{a,\beta,\gamma}(\rho,\pi)(\tilde{\theta}) =
\left(1-\frac{\beta}{\gamma}\right)^{-1}\biggl\{
\log\left[\frac{d\rho}{d\pi_{\exp[-\beta r]}}(\tilde{\theta})\right]
\\
+ \log \int_{\Theta} \pi_{\exp(-\beta r)}(d\theta') \exp
\left[\frac{\beta\gamma}{2N}v_{a,\frac{\gamma}{N}}(\tilde{\theta},\theta')+\beta\Delta_{\frac{\gamma}{a}}(\tilde{\theta},\theta')
\right] \biggr\}.
\end{multline*}
\end{thm}

\subsection{Localized PAC-Bayesian theorems}

\begin{dfn}
From now on, we will deal with model selection. We assume
that we have a family of submodels of $\Theta$: $(\Theta_{i},i\in
I)$ where $I$ is finite or countable. We also choose a probability
measure $\mu\in\mathcal{M}_{+}^{1}(I)$, and assume that we have a
prior distribution $\pi^{i}\in\mathcal{M}_{+}^{1}(\Theta_{i})$ for
every $i$.
\end{dfn}

We choose
$$ \pi = \sum_{i\in I}\mu(i)\pi^{i}_{\exp(-\beta_{i}R)} $$
and apply Theorem \ref{BASICthmdes} that we combine with Theorem
\ref{entropy2} by a union bound argument, to obtain the following
result.

\begin{thm}
\label{thmdes} Let us assume that we have randomized estimators
$(\rho_{i})_{i\in I}$ such that $\rho_{i}(\Theta_{i})=1$, for any
$\varepsilon>0$, for any
$(a,\beta,\beta',\gamma,\gamma',\lambda)\in]0,1]\times(\mathds{R}_{+}^{*})^{5}$
such that $\beta<\gamma$ and $\beta'<\gamma'$, with
$\mathds{P}\bigotimes_{i\in I}\rho_{i}$-probability at least
$1-\varepsilon$ over the sample $(Z_{n})_{n=1,...,N}$ and the
parameters $(\tilde{\theta}_{i})_{i\in I}$, for any $(i,i')\in
I^{2}$ we have:
\begin{multline*}
R_{\frac{\lambda}{a}}\bigl(\tilde{\theta}_{i},\tilde{\theta}_{i'}\bigr)
 \leq
\Phi_{\frac{\lambda}{N}}^{-1} \Biggl\{
r\bigl(\tilde{\theta}_{i}\bigr)-r\bigl(\tilde{\theta}_{i'}\bigr)
+\frac{\lambda}{2N}v_{a,\frac{\lambda}{N}}\bigl(\tilde{\theta}_{i},\tilde{\theta}_{i'}
\bigr)
\\
\shoveright{+\frac{1}{\lambda} \Biggl[
\mathcal{D}_{a,\beta,\gamma}(\rho,\pi^{i})\bigl(\tilde{\theta}_{i}\bigr)
+\mathcal{D}_{a,\beta',\gamma'}(\rho,\pi^{i'})\bigl(\tilde{\theta}_{i'}\bigr)}
\\
+\left(1+\frac{\beta}{\gamma-\beta}+\frac{\beta'}{\gamma'-\beta'}\right)\log\frac{3}{\varepsilon\mu(i)\mu(i')}
\Biggr] \Biggr\}.
\end{multline*}
\end{thm}

In the same way, we can give an integrated variant, using Theorem
\ref{BASICthmrel} and Theorem \ref{entropy1}.

\begin{thm}
\label{thmrel} For any $\varepsilon>0$, for any
$(a,\beta,\beta',\gamma,\gamma',\lambda)\in]0,1]\times(\mathds{R}_{+}^{*})^{5}$
such that $\beta<\gamma$ and $\beta'<\gamma'$, with
$\mathds{P}$-probability at least $1-\varepsilon$, for any
$(i,i')\in I^{2}$ and $(\rho,\rho')
\in\mathcal{M}_{+}^{1}(\Theta_{i})\times
\mathcal{M}_{+}^{1}(\Theta_{i'})$,
\begin{multline*}
\int_{\Theta^{2}} d(\rho\otimes\rho')(\theta,\theta')
R_{\frac{\lambda}{a}}(\theta,\theta')
\\
 \leq
\Phi_{\frac{\lambda}{N}}^{-1} \Biggl\{
\rho(r)-\rho'(r)+\frac{\lambda}{2N}\int_{\Theta^{2}}d(\rho\otimes\rho')(\theta,\theta')
\, v_{a,\frac{\lambda}{N}}(\theta,\theta')
\\
+ \frac{\mathcal{BK}_{a,\beta,\gamma}(\rho,\pi^{i})
+\mathcal{BK}_{a,\beta',\gamma'}(\rho',\pi^{i'})
+\left(1+\frac{\beta}{\gamma-\beta}+\frac{\beta'}{\gamma'-\beta'}\right)\log\frac{3}{\varepsilon\mu(i)\mu(i')}}{\lambda}
\Biggr\}.
\end{multline*}
\end{thm}

\subsection{Choice of the parameters}

In this subsection, we explain how to choose the parameters
$\lambda$, $\beta$, $\beta'$, $\gamma$ and $\gamma'$ in Theorems
\ref{thmdes} and \ref{thmrel}. In some really simple situations
(parametric model with strong assumptions on $\mathds{P}$), this
choice can be made on the basis of theoretical considerations,
however, in many realistic situations, such hypothesis cannot be
made and we would like to optimize the upper bound in the Theorems
with respect to the parameters. This would lead to data-dependant
values for the parameters, and this is not allowed by Theorems
\ref{thmrel} and \ref{thmdes}. Catoni \cite{manuscrit} proposes to
make a union bound on a grid of values of the parameters, thus
allowing optimization with respect to these parameters. We apply
this idea to Theorem \ref{thmrel}, and obtain the following result.

\begin{thm}
\label{thmrel2} Let us choose a measure
$\nu\in\mathcal{M}_{+}^{1}(\Theta)$ that is supported by a finite or
countable set of points, $supp(\nu)$.
Let us assume that we have randomized estimators
$(\rho_{i,\beta})_{i\in I,\beta\in supp(\nu)}$ such that $\rho_{i,\beta}(\Theta_{i})=1$.
For any
$\varepsilon>0$ and $a\in]0,1]$, with
$\mathds{P}\bigotimes_{i\in I,\beta\in supp(\nu)}\rho_{i,\beta}$-probability at least
$1-\varepsilon$ over the sample $(Z_{n})_{n=1,...,N}$ and the
parameters $(\tilde{\theta}_{i,\beta})_{i\in I,\beta\in supp(\nu)}$, for any $(i,i')\in
I^{2}$ and $(\beta,\beta')\in supp(\nu)^{2}$ we have:
\begin{multline*}
R_{\frac{\lambda}{a}} \bigl(\tilde{\theta}_{i,\beta},\tilde{\theta}_{i',\beta'}\bigr)
\\
\leq B\Bigl((i,\beta),(i',\beta')\Bigr)=
 \inf_{
 \tiny{
\begin{array}{c}
\lambda\in]0,+\infty[
\\
\gamma\in]\beta,+\infty[
\\
\gamma'\in]\beta',+\infty[
\end{array}
}
 }
\Phi_{\frac{\lambda}{N}}^{-1} \Biggl\{ r\bigl(\tilde{\theta}_{i,\beta}\bigr)-r\bigl(\tilde{\theta}_{i',\beta'}\bigr)
\\
+\frac{\lambda}{2N}v_{a,\frac{\lambda}{N}}\bigl(\tilde{\theta}_{i,\beta},\tilde{\theta}_{i',\beta'}\bigr)
+ \frac{1}{\lambda}\Biggl[
\mathcal{D}_{a,\beta,\gamma}(\rho_{i,\beta},\pi^{i})\bigl(\tilde{\theta}_{i,\beta}\bigr)
+\mathcal{D}_{a,\beta',\gamma'}(\rho_{i,\beta'},\pi^{i'})\bigl(\tilde{\theta}_{i',\beta'}\bigr)
\\
+\biggl(1+\frac{\beta}{\gamma-\beta}+\frac{\beta'}{\gamma'-\beta'}\biggr)\log\frac{3}{\varepsilon\nu(\lambda)\nu(\gamma)\nu(\beta)\nu(\gamma')\nu(\beta')\mu(i)\mu(i')}
\Biggr] \Biggr\}.
\end{multline*}
\end{thm}

\subsection{Introduction of the complexity function}

It is convenient to remark that we can dissociate the optimization
with respect to the different parameters in Theorem \ref{thmrel2}
thanks to the introduction of an appropriate complexity function.
The model selection algorithm we propose in the next subsection
takes advantage of this decomposition.

\begin{dfn}
Let us choose some real constants $\zeta>1$, $a\in]0,1]$ and
$\varepsilon>0$. We assume that some randomized estimators
$(\rho_{i,\beta})_{i\in I,\beta\in supp(\nu)}$ have been chosen and that
we have drawn $\tilde{\theta}_{i,\beta}$ for every $i\in I$ and $\beta\in
supp(\nu)$. We define, for any $i\in I$,
\begin{multline*}
\mathcal{C}\bigl(i,\beta\bigr)= \inf_{
\gamma\in\left[\zeta\beta,+\infty\right[ } \Biggl\{
\mathcal{D}_{a,\beta,\gamma}(\rho_{i,\beta},\pi^{i})\bigl(\tilde{\theta}_{i,\beta}
\bigr)
\\
+
\left(\frac{\beta}{\gamma-\beta}+\frac{1}{\zeta-1}+1\right)\log\frac{3}{\varepsilon\mu(i)\nu(\beta)\nu(\gamma)}
\Biggr\}.
\end{multline*}
\end{dfn}
\noindent We have the following result.
\begin{thm}
For any $(i,i',\beta,\beta')\in I^{2}\times supp(\nu)^{2}$,
\begin{multline*}
B\Bigl((i,\beta),(i',\beta')\Bigr) \leq \inf_{ \lambda>0 }
\Phi_{\frac{\lambda}{N}}^{-1} \Biggl\{
r\bigl(\tilde{\theta}_{i,\beta}\bigr)-r\bigl(\tilde{\theta}_{i',\beta'}\bigr)
\\
+\frac{\lambda}{2N}v_{a,\frac{\lambda}{N}}\bigl(\tilde{\theta}_{i,\beta},\tilde{\theta}_{i',\beta'}\bigr)
+
\frac{\mathcal{C}\bigl(\tilde{\theta}_{i,\beta}\bigr)+\mathcal{C}\bigl(\tilde{\theta}_{i',\beta'}\bigr)
+\frac{\zeta+1}{\zeta-1}\log\frac{3}{\varepsilon\nu(\lambda)}}{\lambda}
\Biggr\}.
\end{multline*}
\end{thm}

Note, as a consequence of the concavity of
$\Phi_{\frac{\lambda}{N}}^{-1}$, that this implies

\begin{cor}
\label{bornesym}
\begin{multline*}
B\Bigl((i,\beta),(i',\beta')\Bigr) +
B\Bigl((i',\beta'),(i,\beta)\Bigr)
\\
\leq 2 \inf_{ \lambda>0 }
\Phi_{\frac{\lambda}{N}}^{-1} \Biggl\{
\frac{\lambda}{2N}\frac{v_{a,\frac{\lambda}{N}}\bigl(\tilde{\theta}_{i,\beta},\tilde{\theta}_{i',\beta'}\bigr)
+v_{a,\frac{\lambda}{N}}\bigl(\tilde{\theta}_{i',\beta'},\tilde{\theta}_{i,\beta}\bigr)}{2}
\\
+
\frac{\mathcal{C}\left(i,\beta\right)+\mathcal{C}\left(i',\beta'\right)
+\frac{\zeta+1}{\zeta-1}\log\frac{3}{\varepsilon\nu(\lambda)}}{\lambda}
\Biggr\}.
\end{multline*}
\end{cor}

Corollary \ref{bornesym} shows that the symmetric part of $B$ has an upper
bound which contains only variance and complexity terms.

\section{Application: model selection}

\label{application}

In this section, we propose a general algorithm to select among a
family of posteriors - and so to perform model selection as a
particular case. This algorithm was introduced by Catoni
\cite{manuscrit} in the case of classification. We first give the
general form of the estimator. We then give an empirical bound on
its risk. The last subsection is devoted to a theoretical bound
under suitable hypothesis.

\subsection{Selection algorithm}

We introduce the following definition for the sake of simplicity.

\begin{dfn}
Let us put:
$$ \mathcal{P}=\left\{t^{1},...,t^{M}\right\} = \left\{(i,\beta)\in I\times supp(\nu)\right\},$$
where $M=|I|\times|supp(\nu)|$ and the indexation of the $t_{i}$'s
is such that
$$ \mathcal{C}(t^{1}) \leq ... \leq \mathcal{C}(t^{M}). $$
\end{dfn}

Now, remark that there is no reason for the bound $B$ defined in
Theorem \ref{thmrel2} to be sub-additive. So let us define a
sub-additive version of $B$.

\begin{dfn}
We put, for any $(t,t')\in\mathcal{P}^{2}$:
$$ \tilde{B}(t,t')= \inf_{
\tiny{\begin{array}{c} h \geq 1
\\
(t_{0},...,t_{h})\in\mathcal{P}^{h+1}
\\
t_{0}=t,t_{h}=t'
\end{array} } }
\sum_{k=1}^{h}B(t_{k-1},t_{k}).
$$
\end{dfn}

\begin{dfn}
For any $k\in\{1,...,M\}$ we put:
$$ s(k) = \inf \left\{j\in\{1,...,M\},\quad  \tilde{B}(t^{k},t^{j})> 0 \right\} .$$
\end{dfn}
\noindent We are now ready to give the definition of our estimator.
\begin{dfn}
We take as an estimator $\tilde{\theta}_{\hat{t}}$ where
$\hat{t}=t^{\hat{k}}$ and
$$ \hat{k} = \min\left(\arg\max s\right). $$
\end{dfn}

\subsection{Empirical bound on the risk of the selected estimator}

\begin{thm}
 \label{empiricalbound}
 Let us put $\hat{s}=s(\hat{k})$. For any $\varepsilon>0$, with
 $\mathcal{P}\bigotimes_{t\in\mathcal{P}}\rho_{t}$-probability at least $1-\varepsilon$,
 $$ R\left(\tilde{\theta}_{\hat{t}}\right)  \leq R\left(\tilde{\theta}_{t^{j}}\right)
 +
 \begin{cases}
 0, & 1\leq j < \hat{s},
 \\
 \tilde{B}(t^{s(j)},t^{j}) &
 \hat{s}\leq j<\hat{k},
 \\
 \tilde{B}(\hat{t},t^{\hat{s}}) +
 \tilde{B}(t^{\hat{s}},t^{j}), & j\in(\arg\max s)
 \\
 \tilde{B}(\hat{t},t^{j}), & otherwise.
 \end{cases}
 $$
 Thus, adding only non negative terms to the bound,
 $$ R\left(\tilde{\theta}_{\hat{t}}\right)  \leq R\left(\tilde{\theta}_{t^{j}}\right)
 +
 \begin{cases}
 0, & 1\leq j < \hat{s},
 \\
 B(t^{s(j)},t^{j}) + B(t^{j},t^{s(j)}) & \hat{s} \leq j<\hat{k},
 \\
 B(t^{j},t^{\hat{s}})+ B(t^{\hat{s}},t^{j})
 \\
 \qquad  +B(\hat{t},t^{\hat{s}})+ B(t^{\hat{s}},\hat{t}\,)
 & j\in(\arg\max s),
 \\
 B(t^{j},\hat{t}\,)+ B(\hat{t},t^{j}),
 & otherwise.
 \end{cases}
 $$
 \end{thm}

 For a proof, we refer the reader to Catoni \cite{manuscrit} where
 this Theorem is proved in the case of classification, the proof can
 be reproduced here without any modification.

 Theorem \ref{empiricalbound} shows that, according to Corollary \ref{bornesym} (page
 \pageref{bornesym}),  $R(\tilde{\theta}_{\hat{t}})  -
  R(\tilde{\theta}_{t^{j}})$ can be bounded by variance and
 complexity terms relative to posterior distributions with a
 complexity not greater than $\mathcal{C}(t^{j})$, and an
 empirical loss in any case not much larger than the one of
 $\tilde{\theta}_{t^{j}}$.

\subsection{Theoretical bound}

In this subsection, we choose $ \rho_{i,\beta}$ as
$\pi^{i}_{\exp(-\beta r)}$ restricted to a (random) neighborhood of
$\hat{\theta}_{i}$. More formally, for any $p\geq 0$, let us put
$$ \Theta_{i,p} = \left\{\theta\in\Theta_{i},\quad r(\theta) -\inf_{\Theta_{i}} r \leq p\right\}$$
and for any $q\in]0,1]$ let us put
$$ p_{i,\beta}(q) = \inf\left\{p>0,\quad \pi^{i}_{\exp(-\beta r)}(\Theta_{i,p}) \geq q \right\} .$$
Then let us choose $q$ once and for all and let us choose $\rho_{i,\beta}$
so that
$$ \frac{d\rho_{i,\beta}}{d\pi^{i}_{\exp(-\beta r)}}(\theta) =
      \frac{\mathds{1}_{\Theta_{i,p_{i,\beta}(q)}}(\theta)}
        {\pi^{i}_{\exp(-\beta r)}\left(\Theta_{i,p_{i,\beta}(q)}\right)}. $$

Moreover, we assume that $0\leq l_{\theta}(z)\leq C$ for any $\theta\in\Theta$ and $z\in\mathcal{Z}$, and we fix $a=1$.
In this case, note that for any $\lambda\leq N/(2C)$ we have:
$$ v_{1,\frac{\lambda}{N}}(\theta,\theta') \leq \overline{\mathds{P}} \left[\left(l_{\theta}-l_{\theta'}\right)^{2}\right] . $$
For the sake of simplicity we introduce the following definition.

\begin{dfn}
Let us put, for any $(\theta,\theta')\in\Theta^{2}$:
$$ v(\theta,\theta') = \overline{\mathds{P}} \left[\left(l_{\theta}-l_{\theta'}\right)^{2}\right] $$
and
$$ V(\theta,\theta') = \mathds{P} \left[ v(\theta,\theta') \right].$$
\end{dfn}

To obtain the following result we take $\nu$ as the uniform measure
on the grid
$$ supp (\nu) = \left\{2^{0},2^{1},...,2^{\left\lfloor\frac{\log N}{\log 2}\right\rfloor}\right\} .$$

\begin{thm}
\label{theoreticalbound}
Let us put, for any $i\in I$,
$$ \overline{\theta}_{i} =\arg\min_{\theta\in\Theta_{i}} R(\theta)$$
and
$$ \overline{\theta} = \arg\min_{\theta\in\Theta} R(\theta). $$
Let us assume that Mammen and Tsybakov's margin assumption is
satisfied, in other words let there exist
$(\kappa,c)\in[1,+\infty[\times\mathds{R}_{+}^{*}$ such that
$$ \forall \theta\in\Theta,\quad
\left[V(\theta,\overline{\theta})\right]^{\kappa} \leq c \left[R(\theta)-R(\overline{\theta})\right] .$$
Let moreover every sub-model $\Theta_{i},i\in I$ satisfy the
following dimension assumption:
$$ \sup_{\xi\in\mathds{R}} \left\{\xi \left[ \pi^{i}_{\exp(-\xi R)}\left(R\right)-R\left(\overline{\theta}_{i}\right)\right]\right\} \leq d_{i} $$
for a given sequence $(d_{i})_{i\in I}\in(\mathds{R}_{+})^{I}$. Then there is a
constant $\mathcal{C}=\mathcal{C}(\kappa,c,C)$ such that, with
$\mathds{P}\bigotimes_{i\in I,\beta\in
supp(\nu)}\rho_{i,\beta}$-probability at least $1-4\varepsilon$,
\begin{multline*}
R\bigl(\tilde{\theta}_{\hat{t}}\bigr) \leq \inf_{i\in I} \Biggl\{
R\bigl(\,\overline{\theta}_{i}\bigr) + \mathcal{C}
\max\Biggl\{\Biggl(\frac{\left[R\left(\overline{\theta}_{i}\right) -
R\left(\overline{\theta}\right)\right]^{\frac{1}{\kappa}}\left(d_{i}+\log\frac{1}{q}+\log\frac{1+\log_{2}N}{\varepsilon
\mu(i)}\right)}{N}\Biggr)^{\frac{1}{2}},
\\
\left(\frac{d_{i}+\log\frac{1}{q}+\log\frac{1+\log_{2}N}{\varepsilon
\mu(i)}}{N}\right)^{\frac{\kappa}{2\kappa-1}}\Biggr\}\Biggr\}.
\end{multline*}
\end{thm}
\noindent For a proof, see subsection \ref{prooftb} page \pageref{prooftb}.
Let us now make some remarks.

\begin{rmk}[Choice of the parameter $q$]
The better choice for $q$ is obviously $q=1$.
In this case, our estimator is drawn randomly from the distribution,
$$ \rho_{i,\beta} = \pi^{i}_{\exp(-\beta r)}, $$
and the term $\log (1/q)$ vanishes.

However, practitioners worried about the idea to choose randomly in the whole
space an
estimator can use a smaller value of $q$ ensuring that, in any model $i$
and for any $\beta$,
$$ r\left(\tilde{\theta}_{i,\beta}\right) \leq \inf_{\Theta_{i}} r + p_{i,\beta}(q) ,$$
so $\tilde{\theta}_{i,\beta}$ is drawn in a neighborhood of the minimizer
of the empirical risk.
\end{rmk}

\begin{rmk}[Margin assumption]
The so-called margin assumption
$$\left[V(\theta,\overline{\theta})\right]^{\kappa} \leq c \left[R(\theta)-R(\overline{\theta})\right] $$
was first introduced by Mammen and
Tsybakov in the context of classification \cite{MamTsy,Tsy}.
It has however been studied in the context of general regression by Lecue in his
PhD Thesis \cite{Lecue}. The terminology comes from classification,
where a similar assumption can be described in terms of margin. In the general case
however, there is no margin involved, but rather a distance
$V(\theta, \theta')^{1/2}$ on the parameter space, which serves to
describe the shape of the function $R$ in the neighborhood of
its minimum value $R(\overline{\theta})$.
\end{rmk}

\begin{rmk}[Dimension assumption]
In many cases, the assumption
$$ \sup_{\xi\in\mathds{R}} \left\{\xi \left[ \pi^{i}_{\exp(-\xi R)}\left(R\right)-R\left(\overline{\theta}_{i}\right)\right]\right\} \leq d_{i} $$
is just equivalent to the fact that every $\Theta_{i}$ has a finite dimension proportionnal to $d_{i}$.
\end{rmk}

\section{Conclusion}

In this paper we studied a quite general regression problem. We proposed randomized estimators,
that can we drawn in small neighborhoods of empirical minimizers. We proved that these estimators
reach the minimax rate of convergence under Mammen and Tsybakov's margin assumption.

We would like also to point out that the techniques used here can
be applied in a more general context. In particular, Catoni \cite{manuscrit} studied
the transductive classification setting, where for a given $k\in\mathds{N}$, we observe the objects $X_{1}$, $\cdots$, $X_{(k+1)N}$
and the labels $Y_{1}$ ,$\cdots$ ,$Y_{N}$, and we want to predict the $kN$ missing labels $Y_{N+1}$ ,$\cdots$ ,$Y_{(k+1)N}$.
In this context, a deviation result equivalent to Lemma \ref{deviationlemma} (page \pageref{deviationlemma}) can be
proved, and from this result we can obtain a theorem similar to Theorem \ref{empiricalbound} (page \pageref{empiricalbound}).
We refer the reader to our PhD thesis \cite{Alquier2006} for more details (the transductive setting is introduced
page 54 and the deviation result is Lemma 3.1 page 56).

\section{Proofs}

\label{proofs}

\subsection{Proof of Lemma \ref{LEGENDRE}}

\label{PROOFLEGENDRE}

For the sake of completeness, we reproduce here the proof of Lemma
\ref{LEGENDRE} given in Catoni \cite{Classif}.

\begin{proof}[Proof of Lemma \ref{LEGENDRE}]
Let us assume that $h$ is upper-bounded on the support of $n$. Let
us remark that $m$ is absolutely continuous with respect to $n$ if
and only if it is absolutely continuous with respect to
$n_{\exp(h)}$. If it is the case, then
\begin{multline*}
\mathcal{K}\left(m,n_{\exp(h)}\right) = m\left\{\log\left(\frac{d
m}{d n}\right)-h\right\}
       + \log n (\exp\circ h)
\\
 = \mathcal{K}(m,n)-m(h)
       + \log n (\exp\circ h).
\end{multline*}
The left-hand side of this equation is nonnegative and cancels only
for $m=n_{\exp(h)}$. Note that it remains valid when $m$
is not absolutely continuous with respect to $n$ and just says
in this case that $+\infty=+\infty$. We therefore obtain
$$ 0 = \inf_{m\in\mathcal{M}_{+}^{1}(E)} \left[  \mathcal{K}(m,n) - m(h)\right] +\log n (\exp\circ h).$$
This proves the second part of lemma \ref{LEGENDRE}. For the first
part, we do not assume any longer that $h$ is upper bounded on the
support of $n$. We can write
\begin{multline*}
\log n (\exp\circ h) = \sup_{B\in\mathds{R}} \log n [\exp\circ (h
\wedge B)] = \sup_{B\in\mathds{R}} \sup_{m\in\mathcal{M}_{+}^{1}(E)}
\left[m\left(h\wedge B\right)-\mathcal{K}(m,n)\right]
\\
= \sup_{m\in\mathcal{M}_{+}^{1}(E)} \sup_{B\in\mathds{R}}
\left[m\left(h\wedge B\right)-\mathcal{K}(m,n)\right]
\\
= \sup_{m\in\mathcal{M}_{+}^{1}(E)} \left\{\sup_{B\in\mathds{R}}
\left[m\left(h\wedge B\right)\right]-\mathcal{K}(m,n)\right\} =
\sup_{m\in\mathcal{M}_{+}^{1}(E)} \left[m(h)-\mathcal{K}(m,n)\right]
.
\end{multline*}
\end{proof}

\subsection{Proof of Theorem \ref{BASICthmrel}}

\label{proofbasic}

\begin{proof}[Proof of Theorem \ref{BASICthmrel}]
The beginning of this proof follows exactly the proof of Theorem
\ref{BASICthmdes} (page \pageref{BASICthmdes}) until Equation
\ref{intermediaire}. Now, let us apply (to Equation
\ref{intermediaire}) Lemma \ref{LEGENDRE} with
$(E,\mathcal{E})=(\Theta^{2},\mathcal{T}^{\otimes 2})$ to obtain:
\begin{multline*}
\mathds{P} \exp\Biggl\{ \sup_{m\in\mathcal{M}_{+}^{1}(\Theta^{2})}
\Biggl[ \int_{\Theta^{2}} \Biggl\{ \lambda
\Phi_{\frac{\lambda}{N}}\left[
R_{\frac{\lambda}{a}}\left(\theta,\theta'\right)\right]
\\
-\frac{\lambda}{N}\sum_{i=1}^{N}
\Phi_{\frac{\lambda}{N}}\left[\left(\ell_{\theta}-\ell_{\theta'}\right)(Z_{i})\wedge\frac{aN}{\lambda}\right]
\Biggr\} d m(\theta,\theta') - \mathcal{K}(m,\pi\otimes \pi')
\Biggr]\Biggr\} = 1.
\end{multline*}
Consequently
\begin{multline*}
\mathds{P} \exp\Biggl\{
\sup_{(\rho,\rho')\in[\mathcal{M}_{+}^{1}(\Theta)]^{2}} \Biggl[
\int_{\Theta^{2}} \Biggl\{ \lambda \Phi_{\frac{\lambda}{N}}\left[
R_{\frac{\lambda}{a}}\left(\theta,\theta'\right)\right]
\\
-\frac{\lambda}{N}\sum_{i=1}^{N}
\Phi_{\frac{\lambda}{N}}\left[\left(\ell_{\theta}-\ell_{\theta'}\right)(Z_{i})\wedge\frac{aN}{\lambda}\right]
\Biggr\} d (\rho\otimes\rho')(\theta,\theta') -
\mathcal{K}(\rho,\pi)-\mathcal{K}(\rho',\pi') \Biggr]\Biggr\} = 1.
\end{multline*}
This ends the proof.
\end{proof}

\subsection{Proof of Theorems \ref{entropy1} and \ref{entropy2}}

\label{PROOFe1}

\begin{proof}[Proof of Theorem \ref{entropy1}]
First, notice that:
$$
\mathcal{K}\left(\rho,\pi_{\exp(-\beta R)}\right) = \beta\left[\rho
(R)-\pi_{\exp(-\beta R)} (R)\right]
+\mathcal{K}\left(\rho,\pi\right)- \mathcal{K}\left(\pi_{\exp(-\beta
R)},\pi\right).
$$
Let us apply Theorem \ref{BASICthmrel} with $ \pi = \pi' = \rho' =
\pi_{\exp(-\beta R)}$ to obtain with probability at least
$1-\varepsilon$, for any $\rho\in\mathcal{M}_{+}^{1}(\Theta)$,
\begin{multline*}
\mathcal{K}\left(\rho,\pi_{\exp(-\beta R)}\right) \leq \beta \Biggl[
\rho(r)-\pi_{\exp(-\beta R)}(r)
\\
+\frac{\gamma}{2N}
\int_{\Theta^{2}}v_{a,\frac{\gamma}{N}}(\theta,\theta')d\bigl(\,\rho\otimes\pi_{\exp
(-\beta R)}\bigr)(\theta,\theta') + \frac{\log \frac{1}{\varepsilon}
+\mathcal{K}\left(\rho,\pi_{\exp(-\beta R)}\right)}{\gamma}
\\
+
\int_{\Theta^{2}}\Delta_{\frac{\lambda}{a}}(\theta,\theta')d
\bigl(\,\rho\otimes\pi_{\exp
(-\beta R)} \bigr) (\theta,\theta') \Biggr]
+\mathcal{K}\left(\rho,\pi\right)- \mathcal{K}\left(\pi_{\exp(-\beta
R)},\pi\right).
\end{multline*}
Replacing in the right-hand side of this inequality
$\pi_{\exp(-\beta R)}$ with a supremum over all possible
distributions leads to the announced result.
\end{proof}

\begin{proof}[Proof of Theorem \ref{entropy2}]
We have, for any $\theta$:
$$
\log\frac{d\rho}{d\pi_{\exp(-\beta R)}}(\theta) =
\beta\left[R(\theta)-\pi_{\exp(-\beta
R)}(R)\right]+\log\frac{d\rho}{d\pi}(\theta)-\mathcal{K}\left(\pi_{\exp(-\beta
R)},\pi\right).
$$
Let us apply Theorem \ref{BASICthmdes} with $ \pi = \pi' = \rho' =
\pi_{\exp(-\beta R)}$ and a general $\rho$ to obtain with
$\mathds{P}\rho$-probability at least $1-\varepsilon$ over $\theta$,
\begin{multline*}
\log\frac{d\rho}{d\pi_{\exp(-\beta R)}}(\theta) \leq \beta \Biggl[
r(\theta)-\pi_{\exp(-\beta R)}(r)
\\
+\frac{\gamma}{2N}
\int_{\Theta}v_{a,\frac{\gamma}{N}}(\theta,\theta')d\pi_{\exp
(-\beta R)}(\theta') + \frac{\log \frac{1}{\varepsilon}
+\mathcal{K}\left(\rho,\pi_{\exp(-\beta R)}\right)}{\gamma}
\\
+
\int_{\Theta}\Delta_{\frac{\lambda}{a}}(\theta,\theta')d\pi_{\exp
(-\beta R)}(\theta') \Biggr] +\log\frac{d\rho}{d\pi}(\theta)-
\mathcal{K}\left(\pi_{\exp(-\beta R)},\pi\right).
\end{multline*}
The end of the proof is the same as in the case of Theorem \ref{entropy1}.
\end{proof}

\subsection{Proof of Theorem \ref{theoreticalbound}}

\label{prooftb}

We begin by a set of preliminary lemmas and definitions.

\begin{dfn}
For the sake of simplicity, we will write:
$$ r'(\theta,\theta') = r(\theta) - r(\theta') $$
and
$$ R'(\theta,\theta') = R(\theta) - R(\theta') $$
for any $(\theta,\theta')\in\Theta^{2}$.
\end{dfn}

\begin{dfn} \label{margin}
We introduce the margin function:
\begin{align*}
\varphi: \mathds{R}_{+}^{*} & \rightarrow \mathds{R} \\
                          x & \mapsto \sup_{\theta\in\Theta}
                          \biggl[V(\theta,\overline{\theta})-x
                          R'(\theta,\overline{\theta})\biggr].
\end{align*}
\end{dfn}

\begin{lemma}[Mammen and Tsybakov \cite{MamTsy,Tsy}]
\label{marginlemma} Mammen's and Tsybakov margin assumtion:
$$ \exists (\kappa,c)\in[1,+\infty[\times\mathds{R}_{+}^{*},\forall\theta\in\Theta,
\quad V(\theta,\overline{\theta})^{\kappa}
\leq cR'(\theta,\overline{\theta}) $$ implies:
$$\forall x>0,\quad \varphi(x) \leq \left(1-\frac{1}{\kappa}\right)(\kappa c
x)^{- \frac{1}{\kappa-1}} $$
for $\kappa>1$ and $\varphi(c) \leq 0$ for $\kappa=1$.
\end{lemma}

\begin{dfn}
We define the modified Bernstein function:
\begin{align*}
g : \mathds{R} & \rightarrow \mathds{R} \\
             x & \mapsto     \left\{
             \begin{array}{l}
             \frac{2\left[\exp(x)-1-x\right]}{x^{2}} \quad \text{if}
             \quad x \neq 0,
             \\
             \\
             1 \quad \text{if} \quad x = 0.
             \end{array}
             \right.
\end{align*}
\end{dfn}

The function $g$ is a variant of Bernstein's function, used in
Bernstein's inequality (see Bernstein \cite{Bernstein}). Here, we
prove a variant of this inequality.

\begin{lemma}[Variant of Bernstein's inequality]
\label{bern}
We have, for any $\lambda>0$ and any
$(\theta,\theta')\in\Theta^{2}$:
\begin{equation}
\label{bern1}
\mathds{P} \exp\left[\lambda R'(\theta,\theta')-\lambda
r'(\theta,\theta') - \frac{\lambda^{2}}{2N} g\left(\frac{2\lambda
C}{N}\right) V(\theta,\theta') \right]  \leq 1,
\end{equation}
and the reverse inequality
\begin{equation}
\label{bern2}
\mathds{P} \exp\left[\lambda r'(\theta,\theta')-\lambda
R'(\theta,\theta') - \frac{\lambda^{2}}{2N} g\left(\frac{2\lambda
C}{N}\right) V(\theta,\theta') \right]  \leq 1.
\end{equation}
We also have a similar inequality for variances:
\begin{equation}
\label{bern3}
\mathds{P} \exp\left[\frac{N}{4C^{2}} v(\theta,\theta')-\frac{N}{2C^{2}}
V(\theta,\theta') \right]  \leq 1.
\end{equation}
\end{lemma}

\begin{proof}
We have:
\begin{multline*}
\mathds{P} \exp\left[\lambda R'(\theta,\theta')-\lambda
r'(\theta,\theta')\right]
\\
= \exp \left\{\sum_{i=1}^{N} \log \mathds{P} \exp
\left[-\frac{\lambda}{N}\left(l_{\theta}-l_{\theta'}\right)(Z_{i})\right]
+ \lambda R'(\theta,\theta') \right\}.
\end{multline*}
Now, note that for any $b>0$, for any $x\in[-b,b]$ we have:
$$ \exp(-x) = 1 - x + \frac{x^{2}}{2} g(-x) \leq 1 - x + \frac{x^{2}}{2} g(b), $$
so that
$$ \log \mathds{P} \exp
\left[-\frac{\lambda}{N}\left(l_{\theta}-l_{\theta'}\right)(Z_{i})\right]
\leq - \lambda R'(\theta,\theta') + \frac{\lambda^{2}}{2N}
g\left(\frac{2C\lambda}{N}\right)V(\theta,\theta').
$$
It shows that
$$
\mathds{P} \exp\left[\lambda R'(\theta,\theta')-\lambda
r'(\theta,\theta')\right] \leq  \exp \left[\frac{\lambda^{2}}{2N}
g\left(\frac{2C\lambda}{N}\right)V(\theta,\theta')\right].
$$
The proof of the reverse inequality follows the same scheme.
For Inequality \eqref{bern3} note that, using the same scheme,
we obtain:
$$
\mathds{P} \exp\left\{\lambda v(\theta,\theta')-\lambda
V(\theta,\theta') - \frac{\lambda^{2}}{2N} g\left(\frac{4\lambda
C^{2}}{N}\right) \mathds{P}\left[\left(l_{\theta}-l_{\theta'}\right)^{4}(Z)\right] \right\}  \leq 1.
$$
This implies that
$$
\mathds{P} \exp\left[\lambda v(\theta,\theta')-\lambda
V(\theta,\theta') - \frac{\lambda^{2}2C^{2}}{N} g\left(\frac{4\lambda
C^{2}}{N}\right) V(\theta,\theta') \right]  \leq 1.
$$
The choice $\lambda=N/4C^{2}$  and the remark that $g(1)\leq 2$ (actually
$g(1)\simeq 1.4$) leads to Inequality \eqref{bern3}.
\end{proof}

\begin{dfn}
For the sake of shortness, we put:
\begin{multline*}
\delta_{N}(i,q,\varepsilon,\kappa) =
\max\Biggl\{\Biggl(\frac{\left[R\left(\overline{\theta}_{i}\right) -
R\left(\overline{\theta}\right)\right]^{\frac{1}{\kappa}}\left(d_{i}+\log\frac{1}{q}+\log\frac{1+\log_{2}N}{\varepsilon
\mu(i)}\right)}{N}\Biggr)^{\frac{1}{2}},
\\
\left(\frac{d_{i}+\log\frac{1}{q}+\log\frac{1+\log_{2}N}{\varepsilon
\mu(i)}}{N}\right)^{\frac{\kappa}{2\kappa-1}}\Biggr\}.
\end{multline*}
\end{dfn}

Now let us give a brief overview of what follows. Lemma
\ref{lemmaproof1} proves that for some $\beta$,
$\tilde{\theta}_{i,\beta}$ achieves the expected rate of convergence
in model $\Theta_{i}$: $\delta_{N}(i,q,\varepsilon,\kappa)$. As we
then want to use Theorem \ref{empiricalbound} to compare our
estimator $\tilde{\theta}_{\hat{t}}$ to every possible
$\tilde{\theta}_{i,\beta}$, we will have to control the various
parts of the empirical bound $B(.,.)$ by theoretical terms. So we
give two more lemmas: Lemma \ref{lemmaproof2} controls the empirical
variance term $v(.,.)$ by the theoretical variance term $V(.,.)$
while Lemma \ref{lemmaproof3} provides a control for the empirical
complexity term $\mathcal{C}(i,\beta)$. Given these three results we
will be able to prove Theorem \ref{theoreticalbound}. Let us start
with

\begin{lemma} \label{lemmaproof1}
Under the assumptions of Theorem \ref{theoreticalbound}, there is a
constant $\mathcal{C}'=\mathcal{C}'(\kappa,c,C)$ such that, with
$\mathds{P}\bigotimes_{i\in I,\beta\in
supp(\nu)}\rho_{i,\beta}$-probability at least $1-\varepsilon$, for
any $i\in I$, there is a $\beta=\beta^{*}(i)\in supp(\nu)$ such that
$$
R'\left(\tilde{\theta}_{i,\beta},\overline{\theta}_{i}\right) \leq
\mathcal{C}' \delta_{N}(i,q,\varepsilon,\kappa).
$$
\end{lemma}

\begin{proof}
We have, by Inequality \eqref{bern1} in Lemma \ref{bern}:
\begin{multline*}
1 \geq \pi^{i}_{\exp(-\beta R)} \mathds{P} \exp\left[\lambda
R'(.,\overline{\theta}_{i})-\lambda r'(.,\overline{\theta}_{i}) -
\frac{\lambda^{2}}{2N} g\left(\frac{2\lambda C}{N}\right)
V(.,\overline{\theta}_{i}) \right]
\\
\geq \mathds{P} \rho_{i,\beta} \exp\Biggl[\lambda
R'(.,\overline{\theta}_{i})-\lambda r'(.,\overline{\theta}_{i})
\\
- \frac{\lambda^{2}}{2N} g\left(\frac{2\lambda C}{N}\right)
V(.,\overline{\theta}_{i}) -
\log\frac{d\rho_{i,\beta}}{d\pi^{i}_{\exp(-\beta R)}}(.) \Biggr].
\end{multline*}
Thus
\begin{multline*}
\mu(i) \nu(\beta)  \geq \mathds{P}
\rho_{i,\beta} \exp\Biggl[\lambda
R'(.,\overline{\theta}_{i})-\lambda r'(.,\overline{\theta}_{i})
\\
- \frac{\lambda^{2}}{2N} g\left(\frac{2\lambda C}{N}\right)
V(.,\overline{\theta}_{i}) -
\log\frac{d\rho_{i,\beta}}{d\pi^{i}_{\exp(-\beta R)}}(.) +
\log(\mu(i)\nu(\beta)) \Biggr].
\end{multline*}
So, with $\mathds{P}\bigotimes_{i\in I,\beta\in
supp(\nu)}\rho_{i,\beta}$-probability at least $1-\varepsilon/2$, for
any $i\in I$ and $\beta\in supp(\nu)$,
\begin{multline} \label{inter1}
\lambda R'(\tilde{\theta}_{i,\beta},\overline{\theta}_{i}) \leq
\lambda r'(\tilde{\theta}_{i,\beta},\overline{\theta}_{i}) +
\frac{\lambda^{2}}{2N} g\left(\frac{2\lambda C}{N}\right)
V(\tilde{\theta}_{i,\beta},\overline{\theta}_{i})
\\
+ \log\frac{d\rho_{i,\beta}}{d\pi^{i}_{\exp(-\beta
R)}}(\tilde{\theta}_{i,\beta}) + \log\frac{2}{\varepsilon
\mu(i)\nu(\beta)}.
\end{multline}
Note that, using Definition \ref{margin}, for any $x>0$,
$$
V(\tilde{\theta}_{i,\beta},\overline{\theta}_{i}) \leq 2 \biggl[
V(\tilde{\theta}_{i,\beta},\overline{\theta}) +
V(\overline{\theta},\overline{\theta}_{i}) \biggr] \leq 2 \biggl[x
R'(\tilde{\theta}_{i,\beta},\overline{\theta}) + x
R'(\overline{\theta}_{i},\overline{\theta})+2\varphi(x)\biggr].
$$
Therefore Inequality \eqref{inter1} becomes:
\begin{multline*}
\left[\lambda -
\frac{x\lambda^{2}}{N}g\left(\frac{2C\lambda}{N}\right) \right]
R'(\tilde{\theta}_{i,\beta},\overline{\theta}) \leq \left[\lambda +
\frac{x\lambda^{2}}{N}g\left(\frac{2C\lambda}{N}\right) \right]
R'(\overline{\theta}_{i},\overline{\theta})
\\
+ \lambda r'(\tilde{\theta}_{i,\beta},\overline{\theta}_{i}) +
\frac{2 \varphi(x) \lambda^{2}}{N} g\left(\frac{2\lambda
C}{N}\right) +
\beta\left[R\left(\tilde{\theta}_{i,\beta}\right)-\pi^{i}_{\exp(-\beta
R)}R\right]
\\
+ \log\frac{d\rho_{i,\beta}}{d\pi^{i}}(\tilde{\theta}_{i,\beta})
-\mathcal{K}\left(\pi^{i}_{\exp(-\beta R)},\pi^{i}\right) +
\log\frac{2}{\varepsilon \mu(i)\nu(\beta)},
\end{multline*}
leading to
\begin{multline*}
\left[\lambda -
\frac{x\lambda^{2}}{N}g\left(\frac{2C\lambda}{N}\right) - \beta
\right] R'(\tilde{\theta}_{i,\beta},\overline{\theta}) \leq
\left[\lambda +
\frac{x\lambda^{2}}{N}g\left(\frac{2C\lambda}{N}\right) - \beta
\right] R'(\overline{\theta}_{i},\overline{\theta})
\\
+ \frac{2 \varphi(x) \lambda^{2}}{N} g\left(\frac{2\lambda
C}{N}\right) - \beta \pi^{i}_{\exp(-\beta
R)}R'(.,\overline{\theta}_{i}) +
\log\frac{d\rho_{i,\beta}}{d\pi^{i}_{\exp(-\beta
r)}}(\tilde{\theta}_{i,\beta})
\\
- \log\pi^{i}\exp \left[-\lambda r'(.,\overline{\theta}_{i})\right]
-\mathcal{K}\left(\pi^{i}_{\exp(-\beta R)},\pi^{i}\right) +
\log\frac{2}{\varepsilon \mu(i)\nu(\beta)}
\end{multline*}
and
\begin{multline}
\label{inter2} \left[\lambda -
\frac{x\lambda^{2}}{N}g\left(\frac{2C\lambda}{N}\right) -\beta
\right] R'(\tilde{\theta}_{i,\beta},\overline{\theta}_{i}) \leq
\frac{2 x\lambda^{2}}{N}g\left(\frac{2C\lambda}{N}\right)
R'(\overline{\theta}_{i},\overline{\theta})
\\
+ \frac{2 \varphi(x) \lambda^{2}}{N} g\left(\frac{2\lambda
C}{N}\right) - \beta \pi^{i}_{\exp(-\beta
R)}R'(.,\overline{\theta}_{i}) +
\log\frac{d\rho_{i,\beta}}{d\pi^{i}_{\exp(-\beta
r)}}(\tilde{\theta}_{i,\beta})
\\
- \log\pi^{i}\exp \left[-\lambda r'(.,\overline{\theta}_{i})\right]
-\mathcal{K}\left(\pi^{i}_{\exp(-\beta R)},\pi^{i}\right) +
\log\frac{2}{\varepsilon \mu(i)\nu(\beta)}.
\end{multline}
We can then use Inequality \eqref{bern2} (in Lemma \ref{bern}, page
\pageref{bern}) to obtain, with probability at least
$1-\varepsilon/2$, for any $i\in I$ and $\beta\in supp(\nu)$,
\begin{multline}
\label{inter3} - \log\pi^{i}\exp \left[-\lambda
r'(.,\overline{\theta}_{i})\right] \leq \lambda \pi^{i}_{\exp(-\beta
R)} r'(.,\overline{\theta}_{i}) +
\mathcal{K}\left(\pi^{i}_{\exp(-\beta R)},\pi^{i}\right)
\\
\leq \lambda \pi^{i}_{\exp(-\beta R)} R'(.,\overline{\theta}_{i}) +
\frac{\lambda^{2}}{2N}g\left(\frac{2C\lambda}{N}\right)
\pi^{i}_{\exp(-\beta R)} V(.,\overline{\theta}_{i})
\\
\shoveright{+ \mathcal{K}\left(\pi^{i}_{\exp(-\beta
R)},\pi^{i}\right) -\log\frac{\varepsilon\mu(i) \nu(\beta)}{2}}
\\
\leq \left[\lambda
+\frac{x\lambda^{2}}{N}g\left(\frac{2C\lambda}{N}\right)\right]
\pi^{i}_{\exp(-\beta R)} R'(.,\overline{\theta}_{i}) +
\frac{x\lambda^{2}}{N} g\left(\frac{2C\lambda}{N}\right)
R'(\overline{\theta}_{i},\overline{\theta})
\\
+ \frac{2\varphi(x)\lambda^{2}}{N}g\left(\frac{2C\lambda}{N}\right)
+ \mathcal{K}\left(\pi^{i}_{\exp(-\beta R)},\pi^{i}\right)
+\log\frac{2}{\varepsilon\mu(i) \nu(\beta)}.
\end{multline}
Combining Inequalities \eqref{inter2} and \eqref{inter3} we have, with
probability at least $1-\varepsilon$, for any $i$ and $\beta$:
\begin{multline}
\label{repere1}
\left[\lambda -
\frac{x\lambda^{2}}{N}g\left(\frac{2C\lambda}{N}\right) -\beta
\right] R'(\tilde{\theta}_{i,\beta},\overline{\theta}_{i}) \leq
\frac{4x\lambda^{2}}{N}g\left(\frac{2C\lambda}{N}\right)
R'(\overline{\theta}_{i},\overline{\theta})
\\
+ \frac{4 \varphi(x) \lambda^{2}}{N} g\left(\frac{2\lambda
C}{N}\right) + \left[\lambda +
\frac{x\lambda^{2}}{N}g\left(\frac{2C\lambda}{N}\right) -\beta
\right] \pi^{i}_{\exp(-\beta R)}R'(.,\overline{\theta}_{i})
\\
+ \log\frac{d\rho_{i,\beta}}{d\pi^{i}_{\exp(-\beta
r)}}(\tilde{\theta}_{i,\beta}) + 2 \log\frac{2}{\varepsilon
\mu(i)\nu(\beta)}.
\end{multline}
In order to make explicit the terms in Inequality \ref{repere1}, let us
remind the definition of $\rho_{i,\beta}$ in Theorem
\ref{theoreticalbound} (page \pageref{theoreticalbound}) and remark
that
$$ \log\frac{d\rho_{i,\beta}}{d\pi^{i}_{\exp(-\beta
r)}}(\tilde{\theta}_{i,\beta}) \leq \log \frac{1}{q}. $$
Let us also recall the dimension hypothesis in Theorem
\ref{theoreticalbound}, implying that
$$ \pi^{i}_{\exp(-\beta R)}R'(.,\overline{\theta}_{i}) \leq \frac{d_{i}}{\beta} .$$
Let us finally choose $\lambda=2\beta$, Inequality \ref{repere1}
becomes:
\begin{multline}
\label{inter4} \left[\beta -
\frac{16 x\beta^{2}}{N}g\left(\frac{4C\beta}{N}\right)\right]
R'(\tilde{\theta}_{i,\beta},\overline{\theta}_{i}) \leq
\frac{16x\beta^{2}}{N}g\left(\frac{4C\beta}{N}\right)
R'(\overline{\theta}_{i},\overline{\theta})
\\
+ \frac{16 \varphi(x) \beta^{2}}{N} g\left(\frac{4\beta
C}{N}\right) + \left[\beta +
\frac{4 x\beta^{2}}{N}g\left(\frac{4C\beta}{N}\right) \right] d_{i}
\\
+ \log\frac{1}{q} + 2 \log\frac{2}{\varepsilon
\mu(i)\nu(\beta)}.
\end{multline}
Finally, Lemma \ref{marginlemma} together with the margin assumption
in Theorem \ref{theoreticalbound} ensures that
$$ \varphi(x) \leq \left(1-\frac{1}{\kappa}\right)(\kappa c x)^{\frac{-1}{\kappa-1}} $$
if $\kappa>1$ and $\varphi(c) \leq 0$ if $\kappa=1$.
Let us first deal with the case $\kappa=1$. Inequality \eqref{inter4} becomes, taking
$x=c$,
\begin{multline}
\label{repere2}
R'(\tilde{\theta}_{i,\beta},\overline{\theta}_{i})
\leq \left[\frac{1}{2}-
\frac{4c\beta}{N}g\left(\frac{4C\beta}{N}\right) \right]^{-1}
\Biggl\{ \frac{16c\beta}{N}g\left(\frac{4C\beta}{N}\right)
R'(\overline{\theta}_{i},\overline{\theta})
\\
+ \left[1 +
\frac{4 c\beta}{N}g\left(\frac{4C\beta}{N}\right) \right]
\frac{d_{i}}{\beta}
+ \frac{1}{\beta}\log\frac{1}{q} +\frac{2}{\beta} \log\frac{2}{\varepsilon
\mu(i)\nu(\beta)}\Biggr\}.
\end{multline}
In the right-hand side of Inequality \ref{repere2}, the numerator is
optimal for $\beta$ of the order of
$$ \sqrt{\frac{N \left(d_{i} + \log\frac{1}{q} + \log\frac{2}{\varepsilon
\mu(i)\nu(\beta)}\right)}{R'(\overline{\theta}_{i},\overline{\theta})}} $$
but in order to keep the denominator away from zero, the maximal order of magnitude
for $\beta$ is $N$, so let us take $\beta$ of the order of
$$ \min\left\{\sqrt{\frac{N \left(d_{i} + \log\frac{1}{q} + \log\frac{2}{\varepsilon
\mu(i)\nu(\beta)}\right)}{R'(\overline{\theta}_{i},\overline{\theta})}}, N\right\}  .$$
This choice leads to:
\begin{multline}
\label{concl1}
R'(\tilde{\theta}_{i,\beta},\overline{\theta}_{i})
\leq
\mathcal{C}''
\max\Biggl\{\Biggl(\frac{\left[R\left(\overline{\theta}_{i},\overline{\theta}\right)\right]
\left(d_{i}+\log\frac{1}{q}+\log\frac{1+\log_{2}N}{\varepsilon
\mu(i)}\right)}{N}\Biggr)^{\frac{1}{2}},
\\
\left(\frac{d_{i}+\log\frac{1}{q}+\log\frac{1+\log_{2}N}{\varepsilon
\mu(i)}}{N}\right)\Biggr\} = \mathcal{C}''\delta_{N}(i,q,\varepsilon,1)
\end{multline}
for some $\mathcal{C}''=\mathcal{C}''(c,C)$.
In the case where $\kappa>1$, Inequality \eqref{inter4} becomes:
\begin{multline}
\label{inter5} R'(\tilde{\theta}_{i,\beta},\overline{\theta}_{i}) \leq
\left[\frac{1}{2}-
\frac{4 x\beta}{N}g\left(\frac{4C\beta}{N}\right)
\right]^{-1} \Biggl\{
\frac{16x\beta}{N}g\left(\frac{4C\beta}{N}\right)
R'(\overline{\theta}_{i},\overline{\theta})
\\
+ \left(1-\frac{1}{\kappa}\right) \frac{16\beta (\kappa c
x)^{- \frac{1}{\kappa-1}}}{N} g\left(\frac{4\beta C}{N}\right) +
\left[1 +
\frac{4 x\beta}{N}g\left(\frac{4C\beta}{N}\right) \right] \frac{d_{i}}{\beta}
\\
+ \frac{1}{\beta}\log\frac{1}{q} + \frac{2}{\beta} \log\frac{2}{\varepsilon
\mu(i)\nu(\beta)}\Biggr\}.
\end{multline}
Now, we choose $x$ or the order of
$$ \min\left\{\left[R'\left(\overline{\theta}_{i},\overline{\theta}\right)\right]^{-\frac{\kappa-1}{\kappa}},\frac{N}{\beta}\right\}
$$
in Inequality \eqref{inter5} (the case $x=\left[R'\left(\overline{\theta}_{i},\overline{\theta}\right)\right]^{-\frac{\kappa-1}{\kappa}}$
minimizes the numerator while the fact that $x=\mathcal{O}(N/\beta)$ ensures that the denominator does not get too close to zero).
Now, let us consider both cases for $x$, and first
$x=\left[R'\left(\overline{\theta}_{i},\overline{\theta}\right)\right]^{-\frac{\kappa-1}{\kappa}}$.
In this case, let us choose $\beta$ of the order of
$$\min\left\{ \sqrt{\frac{N \left(d_{i} + \log\frac{1}{q} + \log\frac{2}{\varepsilon
\mu(i)\nu(\beta)}\right)}{\left[R'(\overline{\theta}_{i},\overline{\theta})\right]^{\frac{1}{\kappa}}}},N\right\}. $$
This leads to a bound of the order of
\begin{multline*}
\max\Biggl\{\Biggl(\frac{\left[R\left(\overline{\theta}_{i},\overline{\theta}\right)\right]^{\frac{1}{\kappa}}
\left(d_{i}+\log\frac{1}{q}+\log\frac{1+\log_{2}N}{\varepsilon
\mu(i)}\right)}{N}\Biggr)^{\frac{1}{2}},
\\
\frac{d_{i}+\log\frac{1}{q}+\log\frac{1+\log_{2}N}{\varepsilon
\mu(i)}}{N}\Biggr\} \leq \delta_{N}(i,q,\varepsilon,\kappa).
\end{multline*}
In the other case, $x$ is of the order of $N/\beta$ and
$$ \left[R'\left(\overline{\theta}_{i},\overline{\theta}
\right)\right]^{-\frac{\kappa-1}{\kappa}} \geq \frac{N}{\beta},$$
implying that
$$ R'\left(\overline{\theta}_{i},\overline{\theta}\right) \leq \left(\frac{\beta}{N}\right)^{\frac{\kappa}{\kappa-1}} .$$
We have to choose $\beta$ in order to optimize
the numerator, in this case the optimal order of magnitude is
$$ \left[\left(d_{i}+\log\frac{1}{q}+\log\frac{1+\log_{2}N}{\varepsilon
\mu(i)}\right)^{\kappa-1} N \right]^{\frac{1}{2\kappa-1}} $$
and leads to a bound of the order of
$$
\left(\frac{d_{i}+\log\frac{1}{q}+\log\frac{1+\log_{2}N}{\varepsilon
\mu(i)}}{N}\right)^{\frac{\kappa}{2\kappa-1}} \leq \delta_{N}(i,q,\varepsilon,\kappa).
$$
So we have proved that, in the case $\kappa>1$, for some
$\mathcal{C}'''=\mathcal{C}'''(\kappa,c,C)$,
\begin{equation}
\label{concl2}
R'(\tilde{\theta}_{i,\beta},\overline{\theta}_{i}) \leq \mathcal{C}'''\delta_{N}(i,q,\varepsilon,\kappa).
\end{equation}
We put:
$$
\mathcal{C}'(\kappa,c,C)= \left\{
\begin{array}{l}
\mathcal{C}''(c,C) \quad \text{if} \quad \kappa=1
\\
\\
\mathcal{C}'''(\kappa,c,C) \quad \text{if} \quad \kappa>1
\end{array}
\right.
$$
and remark that Inequalities \eqref{concl1} and \eqref{concl2} end the proof.
\end{proof}

\begin{lemma} \label{lemmaproof2}
Under the assumptions of Theorem \ref{theoreticalbound}, with
$\mathds{P}\bigotimes_{i\in I,\beta\in
supp(\nu)}\rho_{i,\beta}$-probability at least $1-\varepsilon$, for
any $(i,i')\in I^{2}$, for any $(\beta,\gamma,\beta',\gamma')\in supp(\nu)^{4}$:
\begin{multline*}
v\left(\tilde{\theta}_{i,\beta},\tilde{\theta}_{i',\beta'}\right)
\leq
2V\left(\tilde{\theta}_{i,\beta},\tilde{\theta}_{i',\beta'}\right)
+
\frac{4C^{2}}{N}\Biggl[\mathcal{D}_{1,\beta,\gamma}\left(\rho_{i,\beta},\pi^{i}\right)\left(\tilde{\theta}_{i,\beta}\right)
\\
+\mathcal{D}_{1,\beta',\gamma'}\left(\rho_{i',\beta'},\pi^{i'}\right)\left(\tilde{\theta}_{i',\beta'}\right)
+\left(1+\frac{\beta}{\gamma-\beta}+\frac{\beta'}{\gamma'-\beta'}\log\frac{3}{\varepsilon\mu(i)\mu(i')}\right)
\Biggr].
\end{multline*}
\end{lemma}

\begin{proof}
According to Inequality \eqref{bern3} (Lemma \ref{bern} page \pageref{bern}),
$$
\mathds{P} \exp\left[\frac{N}{4C^{2}} v(\theta,\theta')-\frac{N}{4C^{2}}
V'(\theta,\theta') \right]  \leq 1.
$$
Let us integrate in $(\theta,\theta')$ with respect to the
distribution $\pi^{i}_{\exp(-\beta R)}\otimes \pi^{i'}_{\exp(-\beta' R)}$
and sum over all $i$, $i'$, $\beta$ and $\beta'$ to obtain, with
$\mathds{P}\bigotimes_{i\in I,\beta\in
supp(\nu)}\rho_{i,\beta}$-probability at least $1-\varepsilon/3$, for
any $(i,i')\in I^{2}$, for any $(\beta,\beta')\in supp(\nu)^{2}$:
\begin{multline*}
v\left(\tilde{\theta}_{i,\beta},\tilde{\theta}_{i',\beta'}\right)
\leq
2 V'\left(\tilde{\theta}_{i,\beta},\tilde{\theta}_{i',\beta'}\right)
\\
+ \frac{4C^{2}}{N}\Biggl\{
  \log\frac{d\rho_{i,\beta}}{d\pi^{i}_{\exp(-\beta R)}}\left(\tilde{\theta}_{i,\beta}\right)
+ \log\frac{d\rho_{i',\beta'}}{d\pi^{i'}_{\exp(-\beta' R)}}\left(\tilde{\theta}_{i',\beta'}\right)
+ \log\frac{3}{\varepsilon}\Biggr\}.
\end{multline*}
To conclude the proof, there remains to  combine this result with
Theorem \ref{entropy2} page \pageref{entropy2},
using a union bound argument.
\end{proof}

\begin{lemma} \label{lemmaproof3}
Under the assumptions of Theorem \ref{theoreticalbound}, there is a
constant $K=K(\kappa,c,C)$ such that, with
$\mathds{P}\bigotimes_{i\in I,\beta\in
supp(\nu)}\rho_{i,\beta}$-probability at least $1-\varepsilon$, for
any $i\in I$, there is $\gamma\in supp(\nu)$ such that, for $\beta=\beta^{*}(i)$,
\begin{equation*}
\mathcal{D}_{1,\beta,\gamma}(\rho_{i,\beta},\pi^{i})\left(\tilde{\theta}_{i,\beta}\right)
\leq
\mathcal{C}(i,\beta)
\leq K \delta_{N}(i,q,\varepsilon,\kappa) \beta .
\end{equation*}
\end{lemma}

\begin{proof}
We have
\begin{multline*}
\mathcal{D}_{1,\beta,\gamma}(\rho_{i,\beta},\pi^{i})\left(\tilde{\theta}_{i,\beta}\right)
\\
=
\left(1-\frac{\beta}{\gamma}\right)^{-1} \Biggl\{
\log \frac{d \rho_{i,\beta}}{d \pi^{i}_{\exp(-\beta r)} } \left(\tilde{\theta}_{i,\beta}\right)
+\log\pi^{i}_{\exp(-\beta r)} \exp
\left[\frac{\beta\gamma}{2N}v\left(.,\tilde{\theta}_{i,\beta}\right)\right]
\Biggr\}
\\
\leq
\left(1-\frac{\beta}{\gamma}\right)^{-1} \Biggl\{ \log\frac{1}{q}
+ \log\pi^{i} \exp
\left[\frac{\beta\gamma}{2N}v\left(.,\tilde{\theta}_{i,\beta}\right) - \beta r'(.,\overline{\theta})\right]
\\
- \log\pi^{i} \exp \left[- \beta r'(.,\overline{\theta}) \right]
\Biggr\}.
\end{multline*}
Let us now apply Lemma \ref{bern} and the now usual integration technique to obtain the following inequalities, with
probability at least $1-4\varepsilon/5$:
\begin{multline*}
-\log\pi^{i} \exp \bigl[ -\beta r'(.,\overline{\theta}) \bigr]
= -\sup_{\rho\in\mathcal{M}_{+}^{1}(\Theta_{i})} \left[ - \beta \rho r'(.,\overline{\theta}) - \mathcal{K}(\rho,\pi^{i})\right]
\\
\leq
-\sup_{\rho\in\mathcal{M}_{+}^{1}(\Theta_{i})} \left[ - \beta \rho R'(.,\overline{\theta})
     + \frac{\beta^{2}}{2N}g\left(\frac{2\beta C}{N}\right)V(.,\overline{\theta})
     + \log\frac{5}{\varepsilon}
     - \mathcal{K}(\rho,\pi^{i})\right]
\\
\leq
-\log\pi^{i} \exp\left(-\beta R'(.,\overline{\theta}) + \frac{\beta^{2}}{2N}g\left(\frac{2\beta C}{N}\right)V(.,\overline{\theta})\right)
+ \log\frac{5}{\varepsilon}.
\end{multline*}
Moreover
\begin{multline*}
\log\pi^{i} \exp
\left[\frac{\beta\gamma}{2N}v\left(.,\tilde{\theta}_{i,\beta}\right) - \beta r'(.,\overline{\theta})\right]
\\
\leq
\log\pi^{i} \exp
\left\{\frac{\beta\gamma}{N}V(.,\tilde{\theta}_{i,\beta}) \beta R'(.,\overline{\theta})
+ \frac{\beta^{2}}{2N}g\left(\frac{2\beta C}{N}\right)V(.,\overline{\theta})\right\}
\\
+ \frac{\beta\gamma 4C^{2}}{N^{2}} \mathcal{D}_{1,\beta,\gamma}(\rho_{i,\beta},\pi^{i})\left(\tilde{\theta}_{i,\beta}\right)
+ \left[1+\frac{4\beta\gamma C^{2}}{N^{2}}+\frac{4\gamma C^{2}\beta^{2}}{N(\gamma-\beta)}\right]\log\frac{5}{\varepsilon},
\end{multline*}
so that
\begin{multline*}
\left[1-\frac{\beta}{\gamma}-
    \frac{\beta\gamma 4C^{2}}{N^{2}} \right] \mathcal{D}_{1,\beta,\gamma}(\rho_{i,\beta},\pi^{i})\left(\tilde{\theta}_{i,\beta}\right)
\\
\leq
\log\frac{1}{q}
+
\log\pi^{i}_{\exp(-\beta R)} \exp\Biggl\{
\left[\frac{\beta\gamma}{N}+\frac{\beta^{2}}{N}g\left(\frac{2\beta C}{N}\right)\right]
V(.,\overline{\theta})
\Biggr\}
\\
\shoveright{+
\left[2+\frac{4\beta\gamma C^{2}}{N^{2}}+\frac{4\gamma C^{2}\beta^{2}}{N(\gamma-\beta)}\right]\log\frac{5}{\varepsilon}}
\\
\leq
\log\frac{1}{q}
+
\log\pi^{i}_{\exp(-\beta R)} \exp\Biggl\{
x \left[\frac{\beta\gamma}{N}+\frac{\beta^{2}}{N}g\left(\frac{2\beta C}{N}\right)\right]
R'(.,\overline{\theta}_{i})
\Biggr\}
\\
\shoveright{+ \left[\frac{2\beta\gamma}{N}+\frac{\beta^{2}}{N}g\left(\frac{2\beta C}{N}\right)\right]
  \left[x R'(\overline{\theta}_{i},\overline{\theta}) +\varphi(x) \right]
+ x \frac{\beta\gamma}{N} R'(\tilde{\theta}_{i,\beta},\overline{\theta}_{i})}
\\
+ \left[2+\frac{4\beta\gamma C^{2}}{N^{2}}+\frac{4\gamma C^{2}\beta^{2}}{N(\gamma-\beta)}\right]\log\frac{5}{\varepsilon}.
\end{multline*}
We then apply Lemma \ref{lemmaproof1} to obtain
with probability at least $1-\varepsilon/5$
$$ R'(\tilde{\theta}_{i,\beta},\overline{\theta}_{i}) \leq \mathcal{C}' \delta_{N}(i,q,\varepsilon/5,\kappa) .$$
Moreover we can choose $\gamma=2\beta$, and remember that the choice $\beta=\beta^{*}(i)$ leads to $\beta<N$, so
\begin{multline}
\label{repere3}
\left[\frac{1}{2}-
    \frac{\beta^{2} 8C^{2}}{N^{2}} \right]
\mathcal{D}_{1,\beta,2\beta}(\rho_{i,\beta},\pi^{i})
\bigl(\tilde{\theta}_{i,\beta}\bigr)
\\
\leq
\log\frac{1}{q}
+
\log\pi^{i}_{\exp(-\beta R)} \exp\Biggl\{
\frac{x \beta^{2}}{N} \left[2+g\left(2 C\right)\right]
R'(.,\overline{\theta}_{i})
\Biggr\}
\\
\shoveright{+ \frac{\beta^{2}}{N} \left[2+g\left(2 C\right)\right]
  \left[x R'(\overline{\theta}_{i},\overline{\theta}) +\varphi(x) \right]
  + \frac{2 x \beta^{2}}{N} \mathcal{C}' \delta_{N}(i,q,\varepsilon/5,\kappa)}
  \\
  + \left[2+\frac{12\beta^{2} C^{2}}{N^{2}}\right]\log\frac{5}{\varepsilon}.
\end{multline}
Now, let us compute:
\begin{multline*}
\log\pi^{i}_{\exp(-\beta R)} \exp\Biggl\{
\frac{x \beta^{2}}{N} \left[2+g\left(2 C\right)\right]
R'(.,\overline{\theta}_{i})
\Biggr\}
\\
\leq
\frac{\beta^{2}}{N} \left[2+g(2C)\right]
\int_{0}^{x} \pi^{i}_{\exp\left\{ -\beta\left[1-\frac{\delta\beta}{N}\left(2+g(2C)\right) \right]\right\}} R'(.,\overline{\theta}_{i})
d\delta
\\
\leq
\frac{ \frac{\beta^{2}}{N} \left[2+g(2C)\right] }{ \beta\left\{1-\frac{x\beta}{N}\left[2+g(2C)\right] \right\}}
x \pi^{i}_{\exp\left\{ -\beta\left[1-\frac{x\beta}{N}\left(2+g(2C)\right) \right]\right\}} R'(.,\overline{\theta}_{i})
\\
\leq
\frac{x d_{i} \beta}{N} \frac{2+g(2C)}{1-\frac{x\beta}{N}\left[2+g(2C)\right]}
\end{multline*}
by the dimension assumption, and so for any $x$ smaller than
$N/\beta$, Inequality \ref{repere3} becomes
\begin{multline}
    \label{xxx}
\left[\frac{1}{2}-
\frac{\beta^{2} 8C^{2}}{N^{2}} \right]
\mathcal{D}_{1,\beta,2\beta}(\rho_{i,\beta},\pi^{i})
\bigl(\tilde{\theta}_{i,\beta}\bigr)
\\
\leq
2 \beta \mathcal{C}' \delta_{N}(i,q,\varepsilon/5,\kappa)
+
\beta
\Biggl\{
\frac{1}{\beta} \log\frac{1}{q}
+
\frac{ d_{i}}{\beta} \frac{2+g(2C)}{1-\frac{x\beta}{N}\left[2+g(2C)\right]}
\\
+ \frac{\beta}{N} \left[2+g\left(2 C\right)\right]
\left[x R'(\overline{\theta}_{i},\overline{\theta}) +\varphi(x) \right]
+ \frac{ 2 + 12 g(2C) }{ \beta }\log\frac{5}{\varepsilon}
\Biggr\}.
\end{multline}
The optimization of the right-hand side of Inequality \eqref{xxx} with respect to $x$ and
$\beta$ leads to the same discussion as for the optimization of the right-hand side of
Inequality \eqref{inter4} (page \pageref{inter4}) in the proof of
Lemma \ref{lemmaproof1} (and a choice of $x$
satisfying $x<N/\beta$).
\end{proof}

\noindent We are now able to proceed to the
\begin{proof}[proof of Theorem \ref{theoreticalbound}]
With $\mathds{P}\bigotimes_{i\in I,\beta\in
supp(\nu)}\rho_{i,\beta}$-probability at least $1-4 \varepsilon$
the inequalities stated in Theorem \ref{empiricalbound}
and in Lemmas \ref{lemmaproof1}, \ref{lemmaproof2} and \ref{lemmaproof3}
are simultaneously satisfied.
In this case, let us choose $i\in I$, $\beta=\beta^{*}(i)$ and $j$ such that $t^{j}=(i,\beta)$.
We have:
$$
R'\bigl(\tilde{\theta}_{\hat{t}},\tilde{\theta}_{(i,\beta)}\bigr)
 \leq
 \begin{cases}
 0, & 1\leq j < \hat{s} \quad \text{(case 1)},
 \\
 B(t^{s(j)},t^{j}) + B(t^{j},t^{s(j)}) & \hat{s} \leq j<\hat{k} \quad \text{(case 2)},
 \\
 B(t^{j},t^{\hat{s}})+ B(t^{\hat{s}},t^{j})
 \\
 \qquad  +B(\hat{t},t^{\hat{s}})+ B(t^{\hat{s}},\hat{t})
 & j\in(\arg\max s) \quad \text{(case 3)},
 \\
 B(t^{j},\hat{t})+ B(\hat{t},t^{j}),
 & \quad \text{otherwise (case 4)}.
 \end{cases}
$$
Let us examine successively the four cases (1, 2, 4 and 3, this last
case being the most difficult).

\textbf{Case 1:} if $1\leq j < \hat{s}$, then
$$ R'\bigl(\tilde{\theta}_{\hat{t}},\tilde{\theta}_{(i,\beta)}\bigr) \leq 0, $$
and so, by the result of Lemma \ref{lemmaproof1} (page
\pageref{lemmaproof1}),
$$ R'\bigl(\tilde{\theta}_{\hat{t}},\overline{\theta}_{i}\bigr) \leq \mathcal{C}'\delta_{N}(i,q,\varepsilon,\kappa) .$$

\textbf{Case 2:} the idea in all the remaining cases (2, 4 and 3) is
that we have to give a control of
$R'\bigl(\tilde{\theta}_{\hat{t}},\tilde{\theta}_{(i,\beta)}\bigr)$,
controlled by the empirical bound $B(.,.)$, in terms of theoretical
quantities only. In case 2, $\hat{s} \leq j<\hat{k}$, then for any
$\lambda\in supp(\nu)$,
\begin{multline*}
R'\bigl(\tilde{\theta}_{t^{s(j)}},\tilde{\theta}_{(i,\beta)}\bigr)
\leq
B(t^{s(j)},t^{j}) + B(t^{j},t^{s(j)})
\\
\leq
\frac{\lambda}{2N} v(t^{s(j)},t^{j})
+ \frac{\mathcal{C}(t^{s(j)})+\mathcal{C}(t^{j})+\frac{\zeta+1}{\zeta-1}\log\frac{3}{\varepsilon \nu(\lambda)}}{\lambda}
\\
\leq \frac{\lambda}{N} V(t^{s(j)},t^{j}) + \frac{\mathcal{C}(t^{j})
+ \mathcal{C}(t^{s(j)})
+\frac{\zeta+1}{\zeta-1}\log\frac{3}{\varepsilon
\nu(\lambda)}}{\lambda}
\\
+ \shoveright{\frac{4C^{2}\lambda}{N^{2}}\Biggl[\mathcal{C}(t^{j}) +
\mathcal{C}(t^{s(j)})
+\left(1+\frac{\beta}{\gamma-\beta}+\frac{\beta'}{\gamma'-\beta'}\log\frac{3}{\varepsilon\mu(i)\mu(i')}\right)
\Biggr]}.
\end{multline*}
As we have, by definition of the function $s(.)$, the inequality
$\mathcal{C}(t^{s(j)})\leq \mathcal{C}(t^{j})$,
\begin{multline*}
R'\bigl(\tilde{\theta}_{t^{s(j)}},\tilde{\theta}_{(i,\beta)}\bigr)
\leq \frac{\lambda}{N} V(t^{s(j)},t^{j}) + \frac{2
\mathcal{C}(t^{j}) +\frac{\zeta+1}{\zeta-1}\log\frac{3}{\varepsilon
\nu(\lambda)}}{\lambda}
\\
+
\frac{4C^{2}\lambda}{N^{2}}\Biggl[2 \mathcal{C}(t^{j})
+\left(1+\frac{\beta}{\gamma-\beta}+\frac{\beta'}{\gamma'-\beta'}\right)\log\frac{3}{\varepsilon\mu(i)\mu(i')}
\Biggr],
\end{multline*}
and so
\begin{multline} \label{step1}
R'\bigl(\tilde{\theta}_{t^{s(j)}},\tilde{\theta}_{t^{j}}\bigr)
\leq
\frac{2\lambda }{N} \left[ x R'\bigl(\tilde{\theta}_{t^{s(j)}},\overline{\theta}\bigr)
+ x R'\bigl(\tilde{\theta}_{t^{j}},\overline{\theta}\bigr) + \varphi(x) \right]
\\\shoveleft{\qquad  + \frac{2 \mathcal{C}(t^{j}) +\frac{\zeta+1}{\zeta-1}
\log\frac{3}{\varepsilon \nu(\lambda)}}{\lambda}}
\\
+
\frac{4C^{2}\lambda}{N^{2}}\Biggl[2 \mathcal{C}(t^{j})
+\left(1+\frac{\beta}{\gamma-\beta}+\frac{\beta'}{\gamma'-\beta'}\right)\log\frac{3}{\varepsilon\mu(i)\mu(i')}\Biggr].
\end{multline}
Thus
\begin{multline*}
\left[1-\frac{2 \lambda  x}{N} \right] R'\bigl(\tilde{\theta}_{t^{s(j)}},\tilde{\theta}_{t^{j}}\bigr)
\leq
\frac{2\lambda }{N} \left[ 2 x R'(\tilde{\theta}_{t^{j}},\overline{\theta}_{i}) + 2 x R'(\overline{\theta}_{i},\overline{\theta}) + \varphi(x) \right]
\\\shoveleft{\qquad \qquad + \frac{2 \mathcal{C}(t^{j})
+\frac{\zeta+1}{\zeta-1}
\log\frac{3}{\varepsilon \nu(\lambda)}}{\lambda}}
\\ + \frac{4C^{2}\lambda}{N^{2}}\Biggl[2 \mathcal{C}(t^{j})
+\left(1+\frac{\beta}{\gamma-\beta}+\frac{\beta'}{\gamma'-\beta'}\right)\log\frac{3}{\varepsilon\mu(i)\mu(i')}\Biggr].
\end{multline*}
Let us apply Lemma \ref{lemmaproof3} page \pageref{lemmaproof3} to upper
bound $\mathcal{C}(t^{j})$, Lemma \ref{lemmaproof1} page \pageref{lemmaproof1}
to upper bound $R'(\tilde{\theta}_{t^{j}},\overline{\theta}_{i})$
and Lemma \ref{marginlemma} page \pageref{marginlemma} to upper bound $\varphi(x)$.
Let us put moreover
$\lambda = \gamma = \gamma' = 2\beta = 2\beta'$ and remember that $\beta<N$.
We obtain, for any $x$ such that $x<N/\beta$,
\begin{multline*}
\left[1-\frac{4 \beta  x}{N} \right] R'\bigl(\tilde{\theta}_{t^{s(j)}},\tilde{\theta}_{t^{j}}\bigr)
\leq
\frac{4\beta }{N} \left[ 2 x R'\bigl(\overline{\theta}_{i},\overline{\theta}\bigr)
+ \left(1-\frac{1}{\kappa}\right)(\kappa c x)^{\frac{-1}{\kappa-1}} \right]
\\\shoveleft{\qquad \qquad + [K(1+32C^{2})+8\mathcal{C}'] \delta_{N}(i,q,\varepsilon,\kappa)
+\frac{1}{2\beta }
+\frac{\zeta+1}{\zeta-1}
\log\frac{3}{\varepsilon \nu(\lambda)}}{2\beta}
\\ + \frac{48C^{2}}{N}\log\frac{3}{\varepsilon\mu(i)\mu(i')}.
\end{multline*}
Let us replace $x$ and $\beta$ by the values given in the discussion for the optimization of the right-hand side of Inequality
\eqref{inter4} (page \pageref{inter4}) in the proof of Lemma \ref{lemmaproof1} (and a choice of $x$
satisfying $x<N/\beta$) to obtain the existence of a constant
$\mathcal{D}'=\mathcal{D}'(\kappa,c,C)$ such that
$$ R'\bigl(\tilde{\theta}_{t^{s(j)}},\tilde{\theta}_{t^{j}}\bigr) \leq \mathcal{D}' \delta_{N}(i,q,\varepsilon,\kappa) .$$
We then deduce from this result and from Lemma \ref{lemmaproof1}
that
$$
R'\bigl(\tilde{\theta}_{\hat{t}},\overline{\theta}_{i}\bigr) \leq
R'\bigl(\tilde{\theta}_{\hat{t}},\tilde{\theta}_{t^{j}}\bigr) +
R'\bigl(\tilde{\theta}_{t^{j}},\overline{\theta}_{i}\bigr)
 \leq
\left(\mathcal{D}'+\mathcal{C}'\right)
\delta_{N}(i,q,\varepsilon,\kappa).
$$

\textbf{Case 4:} the proof follows roughly the same scheme than for
case 2; if $j>\max(\arg\max s)$, note that $\mathcal{C}(t^{j})\geq
\mathcal{C} (\hat{t}\,)$, therefore
\begin{multline*}
R'\bigl(\tilde{\theta}_{\hat{t}},\tilde{\theta}_{(i,\beta)}\bigr)
\leq
B(\hat{t},t^{j}) + B(t^{j},\hat{t})
\\
\leq
\frac{\lambda}{2N} v(\hat{t},t^{j})
+ \frac{2\mathcal{C}(t^{j})+\frac{\zeta+1}{\zeta-1}\log\frac{3}{\varepsilon \nu(\lambda)}}{\lambda}
\\
\leq
\frac{\lambda}{N} V(\hat{t},t^{j})
+ \frac{2\mathcal{C}(t^{j})+\frac{\zeta+1}{\zeta-1}\log\frac{3}{\varepsilon \nu(\lambda)}}{\lambda}
\\
\shoveright{
+ \frac{4C^{2}\lambda}{N^{2}}\Biggl[2 \mathcal{C}(t^{j})
+\left(1+\frac{\beta}{\gamma-\beta}+\frac{\beta'}{\gamma'-\beta'}\log\frac{3}{\varepsilon\mu(i)\mu(i')}\right)}
\\
\leq
\frac{2\lambda}{N} \left[x R'(\tilde{\theta}_{\hat{t}},\overline{\theta})
       + x R'(\tilde{\theta}_{t^{j}},\overline{\theta}) + \varphi(x)\right]
+ \frac{2\mathcal{C}(t^{j})+\frac{\zeta+1}{\zeta-1}\log\frac{3}{\varepsilon \nu(\lambda)}}{\lambda}
\\
+ \frac{4C^{2}\lambda}{N^{2}}\Biggl[2 \mathcal{C}(t^{j})
+\left(1+\frac{\beta}{\gamma-\beta}+\frac{\beta'}{\gamma'-\beta'}\log\frac{3}{\varepsilon\mu(i)\mu(i')}\right)\Biggr].
\end{multline*}
Thus
\begin{multline*}
\left[1-\frac{2\lambda x}{N}\right] R'\bigl(\tilde{\theta}_{\hat{t}},\tilde{\theta}_{(i,\beta)}\bigr)
\leq
\frac{4\lambda x}{N} \left[
R'\bigl(\tilde{\theta}_{(i,\beta)},\overline{\theta}_{i}\bigr)
+ R'\left(\overline{\theta}_{i},\overline{\theta}\right)
\right]
\\
+
\frac{2\lambda \varphi(x)}{N}
+ \frac{2\mathcal{C}(t^{j})+\frac{\zeta+1}{\zeta-1}\log\frac{3}{\varepsilon \nu(\lambda)}}{\lambda}
\\
+ \frac{4C^{2}\lambda}{N^{2}}\Biggl[2 \mathcal{C}(t^{j})
+\left(1+\frac{\beta}{\gamma-\beta}+\frac{\beta'}{\gamma'-\beta'}\log\frac{3}{\varepsilon\mu(i)\mu(i')}\right)
\Biggr].
\end{multline*}
Let us apply Lemma \ref{lemmaproof3} page \pageref{lemmaproof3} to upper
bound $\mathcal{C}(t^{j})$,
Lemma \ref{lemmaproof1} page \pageref{lemmaproof1} to upper bound
$R'(\tilde{\theta}_{t^{j}},\overline{\theta}_{i})$
and Lemma \ref{marginlemma} page \pageref{marginlemma} to upper bound $\varphi(x)$.
Let us put moreover $\lambda = \gamma = \gamma' = 2\beta = 2\beta'$ and remember that $\beta<N$.
We obtain, for any $x$ such that $x<N/\beta$,
\begin{multline*}
\left[1-\frac{4 \beta  x}{N} \right] R'\bigl(\tilde{\theta}_{t^{s(j)}},\tilde{\theta}_{(i,\beta)}\bigr)
\leq
\frac{4\beta }{N} \left[ 2 x R'\bigl(\overline{\theta}_{i},\overline{\theta}\bigr)
+ \left(1-\frac{1}{\kappa}\right)(\kappa c x)^{\frac{-1}{\kappa-1}} \right]
\\\shoveleft{\qquad \qquad + [K(1+32C^{2})+8\mathcal{C}'] \delta_{N}(i,q,\varepsilon,\kappa)
+\frac{1}{2\beta }
+\frac{\zeta+1}{\zeta-1}
\log\frac{3}{\varepsilon \nu(\lambda)}}{2\beta}
\\ + \frac{48C^{2}}{N}\log\frac{3}{\varepsilon\mu(i)\mu(i')}.
\end{multline*}
Choosing $x$ exactly in the same way as in the previous cases
and replacing $\beta=\beta^*(i)$ with its value,
we obtain the existence of
$\mathcal{D}''=\mathcal{D}''(\kappa,c,C)$ such that
$$
R'\bigl(\tilde{\theta}_{\hat{t}},\tilde{\theta}_{(i,\beta)}\bigr)
\leq
\mathcal{D}''\delta_{N}(i,q,\varepsilon,\kappa)
$$
and so
$$
R\bigl(\tilde{\theta}_{\hat{t}},\overline{\theta}_{i}\bigr)
\leq \left(\mathcal{C}'+\mathcal{D}''\right)\delta_{N}(i,q,\varepsilon,\kappa).
$$

\textbf{Case 3:} if $j\in(\arg\max s)$, remember that $\hat{s} = s(\hat{t}) = s(j)$,
so that
\begin{equation} \label{lasteq}
R'\bigl(\tilde{\theta}_{\hat{t}},\tilde{\theta}_{t^{j}}\bigr)
\leq
\Bigl[B(t^{j},t^{s(j)})+ B(t^{s(j)},t^{j})\Bigr]
+
\Bigl[B(\hat{t},t^{\hat{s}})+ B(t^{\hat{s}},\hat{t})\Bigr].
\end{equation}
We are going to upper bound separately $B(t^{j},t^{s(j)})+
B(t^{s(j)},t^{j})$ and $B(\hat{t},t^{\hat{s}})+
B(t^{\hat{s}},\hat{t})$. Let us first deal with the term
$B(t^{j},t^{s(j)})+ B(t^{s(j)},t^{j})$:
\begin{multline}
\label{repere4}
\Bigl[B(t^{j},t^{s(j)})+ B(t^{s(j)},t^{j})\Bigr] \leq
\frac{\lambda}{2N} v(t^{s(j)},t^{j}) +
\frac{2\mathcal{C}(t^{j})+\frac{\zeta+1}{\zeta-1}\log\frac{3}{\varepsilon
\nu(\lambda)}}{\lambda}
\\
\leq
\frac{\lambda}{N} V(t^{s(j)},t^{j})
+ \frac{2\mathcal{C}(t^{j})+\frac{\zeta+1}{\zeta-1}\log\frac{3}{\varepsilon \nu(\lambda)}}{\lambda}
\\
\shoveright{
+ \frac{4C^{2}\lambda}{N^{2}}\Biggl[2 \mathcal{C}(t^{j})
+\left(1+\frac{\beta}{\gamma-\beta}+\frac{\beta'}{\gamma'-\beta'}\log\frac{3}{\varepsilon\mu(i)\mu(i')}\right)}
\\
\leq
\frac{2\lambda}{N} \left[x R'(\tilde{\theta}_{t^{s(j)}},\tilde{\theta}_{t^{j}})
+ 2 x R'(\tilde{\theta}_{t^{j}},\overline{\theta}_{i}) + 2 x R'(\overline{\theta}_{i},\overline{\theta})+ \varphi(x)\right]
\\
+ \frac{2\mathcal{C}(t^{j})+\frac{\zeta+1}{\zeta-1}\log\frac{3}{\varepsilon \nu(\lambda)}}{\lambda}
\\
+ \frac{4C^{2}\lambda}{N^{2}}\Biggl[2 \mathcal{C}(t^{j})
+\left(1+\frac{\beta}{\gamma-\beta}+\frac{\beta'}{\gamma'-\beta'}\log\frac{3}{\varepsilon\mu(i)\mu(i')}\right)\Biggr].
\end{multline}
Let us notice that
$$
R'(\tilde{\theta}_{t^{s(j)}},\tilde{\theta}_{t^{j}}) \leq B(t^{s(j)},t^{j})
$$
and remember that, by definition, $B(t^{j},t^{s(j)})\geq 0$. This shows that
$$
R'(\tilde{\theta}_{t^{s(j)}},\tilde{\theta}_{t^{j}})
\leq \Bigl[B(t^{j},t^{s(j)})+ B(t^{s(j)},t^{j})\Bigr].
$$
Once again, let us apply Lemma \ref{lemmaproof3} to upper bound
$\mathcal{C}(t^{j})$, Lemma \ref{lemmaproof1} to upper bound
$R'(\tilde{\theta}_{t^{j}},\overline{\theta}_{i})$ and Lemma
\ref{marginlemma} to upper bound $\varphi(x)$. Let us put moreover
$\lambda = \gamma = \gamma' = 2\beta = 2\beta'$. Inequality
\ref{repere4} becomes:
\begin{multline*}
\left(1-\frac{4\beta x}{N} \right)
\Bigl[B(t^{j},t^{s(j)})+ B(t^{s(j)},t^{j})\Bigr]
\leq
\frac{4\beta }{N} \left[ 2 x R'(\overline{\theta}_{i},\overline{\theta})
+ \left(1-\frac{1}{\kappa}\right)(\kappa c x)^{\frac{-1}{\kappa-1}} \right]
\\\shoveleft{\qquad \qquad + [K(1+32C^{2})+8\mathcal{C}'] \delta_{N}(i,q,\varepsilon,\kappa)
+\frac{1}{2\beta }
+\frac{\zeta+1}{\zeta-1}
\log\frac{3}{\varepsilon \nu(\lambda)}}{2\beta}
\\ + \frac{48C^{2}}{N}\log\frac{3}{\varepsilon\mu(i)\mu(i')}
\end{multline*}
and therefore
$$ \Bigl[B(t^{j},t^{s(j)})+ B(t^{s(j)},t^{j})\Bigr] \leq \mathcal{E} \delta_{N}(i,q,\varepsilon,\kappa). $$
There remains to upper bound
$\Bigl[B(\hat{t},t^{\hat{s}})+ B(t^{\hat{s}},\hat{t})\Bigr]$.
We will use to that purpose the fact that $\mathcal{C}(\hat{t})\leq \mathcal{C}(t^{j})$:
\begin{multline*}
\Bigl[B(\hat{t},t^{\hat{s}})+ B(t^{\hat{s}},\hat{t})\Bigr]
\\
\leq
\frac{2\lambda}{N} \left[x R'(\tilde{\theta}_{\hat{t}},\tilde{\theta}_{t^{j}})+
x R'(\tilde{\theta}_{t^{s(j)}},\tilde{\theta}_{t^{j}})
+ 2 x R'(\tilde{\theta}_{t^{j}},\overline{\theta}_{i}) + 2 x R'(\overline{\theta}_{i},\overline{\theta})
+
\varphi(x)\right]
\\
+ \frac{2\mathcal{C}(t^{j})+\frac{\zeta+1}{\zeta-1}\log\frac{3}{\varepsilon \nu(\lambda)}}{\lambda}
\\
+ \frac{4C^{2}\lambda}{N^{2}}\Biggl[2 \mathcal{C}(t^{j})
+\left(1+\frac{\beta}{\gamma-\beta}+\frac{\beta'}{\gamma'-\beta'}\log\frac{3}{\varepsilon\mu(i)\mu(i')}\right)\Biggr].
\end{multline*}
Note that we have already proved that
$$
R'(\tilde{\theta}_{t^{s(j)}},\tilde{\theta}_{t^{j}}) \leq \Bigl[B(t^{j},t^{s(j)})+ B(t^{s(j)},t^{j})\Bigr]
\leq \mathcal{E} \delta_{N}(i,q,\varepsilon,\kappa).
$$
Plugging all these results into Inequality \eqref{lasteq}, we obtain,
\begin{multline*}
\left(1-\frac{2\lambda x }{N }\right) R'\bigl(\tilde{\theta}_{\hat{t}},\tilde{\theta}_{t^{j}}\bigr)
\leq
\mathcal{E} \delta_{N}(i,q,\varepsilon,\kappa)
\\
+
\frac{2\lambda}{N} \left[
x \mathcal{E} \delta_{N}(i,q,\varepsilon,\kappa)
+ 2 x R'(\tilde{\theta}_{t^{j}},\overline{\theta}_{i}) + 2 x R'(\overline{\theta}_{i},\overline{\theta})
+
\varphi(x)\right]
\\
+ \frac{2\mathcal{C}(t^{j})+\frac{\zeta+1}{\zeta-1}\log\frac{3}{\varepsilon \nu(\lambda)}}{\lambda}
\\
+ \frac{4C^{2}\lambda}{N^{2}}\Biggl[2 \mathcal{C}(t^{j})
+\left(1+\frac{\beta}{\gamma-\beta}+\frac{\beta'}{\gamma'-\beta'}\log\frac{3}{\varepsilon\mu(i)\mu(i')}\right)\Biggr].
\end{multline*}
As usual, let us apply Lemma \ref{lemmaproof3} to upper bound $\mathcal{C}(t^{j})$,
Lemma \ref{lemmaproof1} to upper bound
$R'(\tilde{\theta}_{t^{j}},\overline{\theta}_{i})$ and
Lemma \ref{marginlemma} to upper bound $\varphi(x)$. Let us
put $\lambda = \gamma = \gamma' = 2\beta = 2\beta'$, to obtain
\begin{multline*}
\left(1-\frac{4\beta x}{N} \right)
R'\left(\tilde{\theta}_{\hat{t}},\tilde{\theta}_{t^{j}}\right)
\leq
\frac{4\beta }{N} \left[ 2 x R'(\overline{\theta}_{i},\overline{\theta})
+ \left(1-\frac{1}{\kappa}\right)(\kappa c x)^{\frac{-1}{\kappa-1}} \right]
\\\shoveleft{\qquad \qquad + [K(1+32C^{2})+8\mathcal{C}'+3\mathcal{E}] \delta_{N}(i,q,\varepsilon,\kappa)
+\frac{1}{2\beta }
+\frac{\zeta+1}{\zeta-1}
\log\frac{3}{\varepsilon \nu(\lambda)}}{2\beta}
\\ + \frac{48C^{2}}{N}\log\frac{3}{\varepsilon\mu(i)\mu(i')},
\end{multline*}
and therefore
$$ R'\left(\tilde{\theta}_{\hat{t}},\tilde{\theta}_{t^{j}}\right) \leq \mathcal{E}'\delta_{N}(i,q,\varepsilon,\kappa) .$$
This ends the proof.
\end{proof}

\section*{Appendix : bounding the effect of truncation}

We will show here how to upper bound
$R(\theta)-R(\theta')-R_{\lambda}(\theta,\theta')$ by some quantity
$\Delta_{\lambda}(\theta,\theta')$ depending on an additional
hypothesis on the data distribution.

\begin{lemma}
Let us assume that we are in the i.i.d. case, where $P_{1}=...=P_{N}$
and that for some constants $(b,B)\in\mathds{R}_{+}^{2}$
$$ \forall \theta\in\Theta,\quad P_{1} \left\{\exp\left[b\left|l_{\theta}(Z_{1}) \right|\right]\right\}
\leq B.$$ Then, for any $(\theta,\theta')\in\Theta^{2}$,
$$ R(\theta)-R(\theta')-R_{\lambda}(\theta,\theta') \leq
  \Delta_{\lambda}(\theta,\theta') = \frac{2 B}{b}\exp\left(\frac{-bN}{2 \lambda}\right). $$
\end{lemma}

\begin{proof}
From definitions,
\begin{multline*}
R(\theta)-R(\theta')-R_{\lambda}(\theta,\theta')
\\
=
P_{1}\left\{l_{\theta}(Z_{1})-l_{\theta'}(Z_{1})-\left[l_{\theta}(Z_{1})-l_{\theta'}(Z_{1})\right]\wedge\frac{N}{\lambda}\right\}
\\
=
P_{1}\left[\left(l_{\theta}(Z_{1})-l_{\theta'}(Z_{1})
-\frac{N}{\lambda}\right)_{+}\right],
\end{multline*}
where $(x)_{+}=x\wedge 0$. So we can write
\begin{multline*}
R(\theta)-R(\theta')-R_{\lambda}(\theta,\theta')
\\
\leq
\int_{0}^{+\infty}
P_{1}\left[\left(l_{\theta}(Z_{1})-l_{\theta'}(Z_{1})-\frac{N}{\lambda}\right)_{+}>t
\right] dt
\\
\leq \int_{0}^{+\infty}
P_{1}\left[l_{\theta}(Z_{1})-l_{\theta'}(Z_{1})-\frac{N}{\lambda}>t
\right] dt
\\
\leq \int_{0}^{+\infty}
P_{1}\left\{\exp\left[\frac{b}{2}\left(l_{\theta}(Z_{1})-l_{\theta'}(Z_{1})-\frac{N}{\lambda}-t\right)
\right]\right\} dt
\\
\leq \exp\left(\frac{-bN}{2\lambda}\right) B \int_{0}^{+\infty}
\exp\left(-\frac{bt}{2}\right) dt,
\end{multline*}
leading to the result stated in the lemma.
\end{proof}

\label{annex}

\bibliographystyle{acm}
\bibliography{biblio2}

\end{document}